\documentclass[10pt,twocolumn,letterpaper]{article}

\usepackage{iccv}
\usepackage{times}
\usepackage{epsfig}
\usepackage{graphicx}
\usepackage{amsmath}
\usepackage{amssymb}

\usepackage{threeparttable}
\usepackage{diagbox}
\usepackage{amsmath}
\usepackage{booktabs}
\usepackage{color}
\usepackage{subfigure}
\usepackage{booktabs}
\usepackage{threeparttable}
\usepackage{multirow}
\usepackage{amsthm}
\usepackage{mathrsfs}
\usepackage{algorithm}
\usepackage{algorithmic}
\usepackage{lettrine}

\def\comp{\ensuremath\mathop{\scalebox{.7}{$\circ$}}}
\newtheorem{assumption}{Assumption}
\newtheorem{corollary}{Corollary}

\usepackage[title]{appendix}

\usepackage[colorlinks,breaklinks=true,bookmarks=false]{hyperref}
\usepackage{xspace}

\iccvfinalcopy 

\def\iccvPaperID{****} 

\ificcvfinal\fi

\begin{document}

\title{Improving Robustness of Adversarial Attacks Using an Affine-Invariant Gradient Estimator}

\author{Wenzhao~Xiang$^{1 \ast}$, Hang~Su$^{2 \ast}$, Chang~Liu$^{1}$, Yandong~Guo$^{3}$, Shibao~Zheng$^{1}$\\
$^{1}$ Shanghai Jiao Tong University, 
$^{2}$ Tsinghua University, 
$^{3}$ OPPO Research Institute \\
{\tt\small \{690295702,sunrise6513,sbzh\}@sjtu.edu.cn},\\
{\tt\small suhangss@mail.tsinghua.edu.cn}, {\tt\small guoyandong@oppo.com}
}

\maketitle
\ificcvfinal\thispagestyle{empty}\fi

\begin{abstract}
   As designers of artificial intelligence try to outwit hackers, both sides continue to hone in on AI's inherent vulnerabilities. Designed and trained from certain statistical distributions of data, AI's deep neural networks (DNNs) remain vulnerable to deceptive inputs that violate a DNN’s statistical, predictive assumptions. Before being fed into a neural network, however, most existing adversarial examples cannot maintain malicious functionality when applied to an affine transformation. For practical purposes, maintaining that malicious functionality serves as an important measure of the robustness of adversarial attacks. To help DNNs learn to defend themselves more thoroughly against attacks, we propose an affine-invariant adversarial attack, which can consistently produce more robust adversarial examples over affine transformations. For efficiency, we propose to disentangle current affine-transformation strategies from the Euclidean geometry coordinate plane with its geometric translations, rotations and dilations; we reformulate the latter two in polar coordinates. Afterwards, we construct an affine-invariant gradient estimator by convolving the gradient at the original image with derived kernels, which can be integrated with any gradient-based attack methods. Extensive experiments on ImageNet, including some experiments under physical condition, demonstrate that our method can significantly improve the affine invariance of adversarial examples and, as a byproduct, improve the transferability of adversarial examples, compared with alternative state-of-the-art methods.
  \footnote{The paper is under consideration at Computer Vision and Image Understanding.}
\end{abstract}

\section{Introduction}
Deep neural networks have been widely used in image recognition~\cite{he2016deep,simonyan2014very}, medical image analysis~\cite{shen2017deep,litjens2017survey}, autonomous driving~\cite{al2017deep,grigorescu2020survey}, \etc. However, recent research shows that deep neural networks remain highly vulnerable to adversarial examples~\cite{szegedy2013intriguing,biggio2013evasion,goodfellow2014explaining}, since the accuracy of image recognition may degenerate significantly with the addition of small perturbations. No wonder then that more and more attention is being paid to the existence of adversarial examples that may cause safety problems.

\begin{figure}[tp]
  \centering
  \includegraphics[width=\columnwidth]{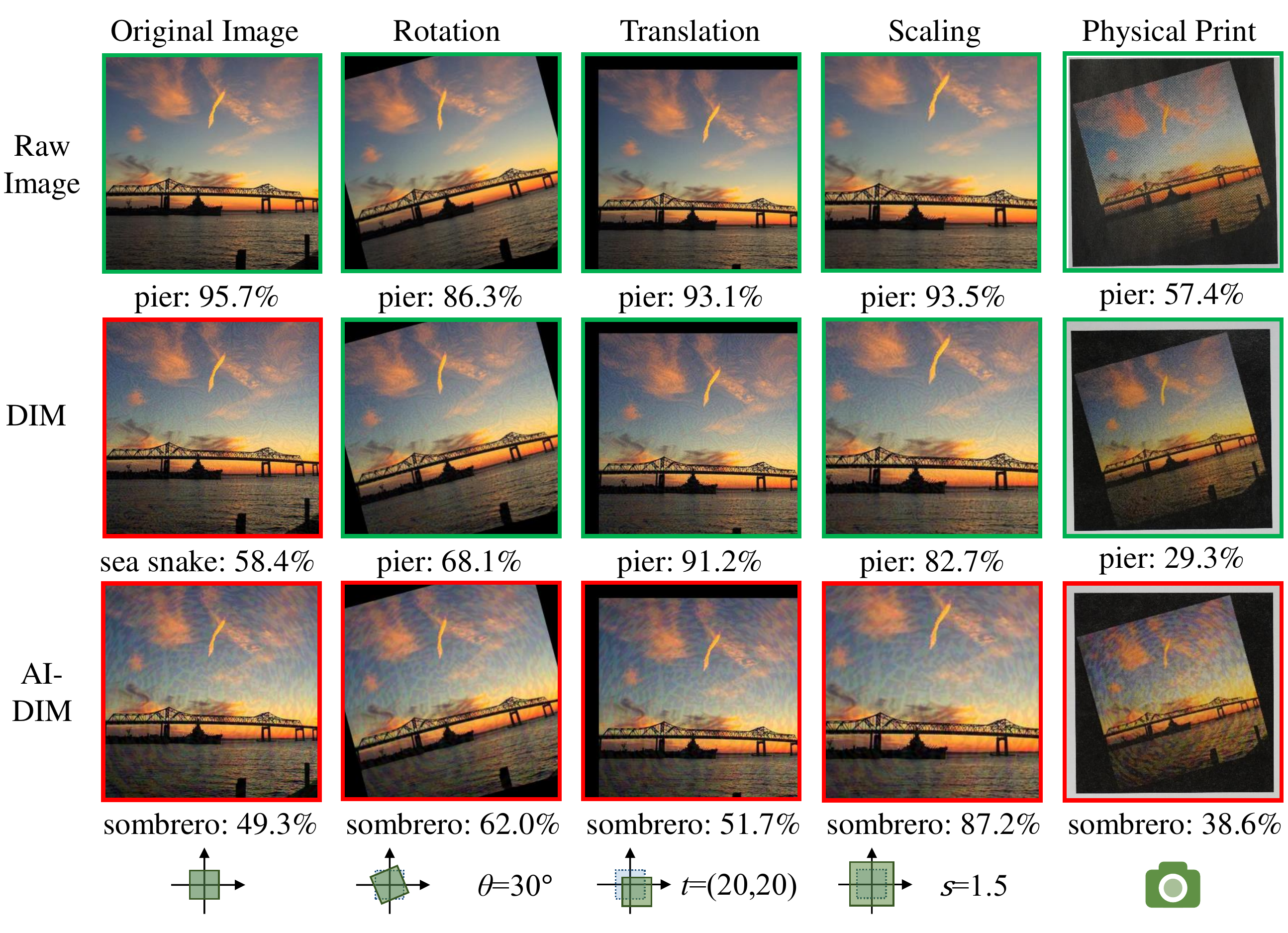} 
  \caption{The adversarial examples generated by the diversity input method (DIM)~\cite{xie2019improving} and the proposed affine-invariant DIM (AI-DIM) for the Inception-v3~\cite{szegedy2016rethinking} model. The images in the red box are misclassified and those in the green box are correctly classified. 
The notation $\theta$ means rotation angle, $t$ means translation offsets and $s$ means scaling factor.
  The proposed method shows better affine invariance under different transformations in terms of rotation, translation and scaling.}
  \vspace{-2ex}
  \label{fig:1} 
\end{figure}

Various adversarial attack methods have been proposed to generate robust and imperceptible adversarial examples, including the Fast Gradient Sign Method (FGSM)~\cite{goodfellow2014explaining}, Projected Gradient Descent (PGD)~\cite{madry2017towards}, and Carlini \& Wagner’s method (C\&W)~\cite{carlini2017towards}, \etc. 
However, most of these algorithms have not considered the affine transformation to the input images, which may influence the robustness of attacks~\cite{athalye2018synthesizing}.
As shown in Fig.~\ref{fig:1}, the resultant adversarial images generally fail to evade the classifier under affine transformation, which limits the classifier's relevance to more practical risks.  
Therefore, it is imperative to generate adversarial examples that can maintain the malicious functionality required to fool the classifier under affine transformation, 
serving as a good measurement for the robustness of adversarial attacks. 

Previous works, such as Expectation Over Transformation (EOT)~\cite{athalye2018synthesizing} and Robust Physical Perturbations (RP$_2$)~\cite{eykholt2018robust}, sample from the preset transformation distribution for estimation to make adversarial examples more robust. However, most of these works did not formally build a generic affine model, which may degrade the performance of attack methods. Besides, the sampling process adds high computational complexity, yielding low generation efficiency of adversarial examples. 

\subsection{Our Proposal}
To address the aforementioned issues, we propose a novel method with an affine-invariant gradient estimator to generate more robust adversarial examples against a general affine transformation. 
Specifically, we formulate the attack problem as an optimization to maximize the expectation of adversarial loss over the affine transformation. We generate the adversarial perturbations through an ensemble of images composed of a legitimate one and its affine-transformed versions. Afterwards, we decompose the affine transformation into translation, rotation, and scaling, and derive their transformation invariance.

To improve the calculation efficiency of the gradient estimator, we derive a kernel-based estimator to approximate the affine-invariant gradient by convolving the original gradient with the specific kernels. 
As to the translation transformation, we implement the convolution operation referred to as the Translation-Invariant Method (TI)~\cite{dong2019evading}. For rotation and scaling with higher complexity, we transfer the expectation of different rotation and scaling transformations into convolution in polar space. Theoretical analysis shows that the rotation and scaling invariance can be approximately equivalent to the translation invariance in polar space. By combining our method with any of the gradient-based attack methods (\eg, FGSM~\cite{goodfellow2014explaining}, PGD~\cite{madry2017towards}, \etc.), we can obtain adversarial examples that are more robust and transferable to affine transformation, and with relatively lower computational cost.

Additionally, as an enhancement of TI~\cite{dong2019evading}, the proposed attack further improves the input diversity, which means better transferability for defense models according to \cite{xie2019improving,dong2019evading}. Therefore, when set as the initialization for query-based black-box attacks, the proposed method can further improve the attack success rate and reduce the required queries. 

Experiments on the ImageNet dataset~\cite{russakovsky2015imagenet} validate the effectiveness of the proposed method. 
Our best method improves the attack success rate by $35.5\%$ and saves about $99\%$ on the computation cost, compared to EOT~\cite{athalye2018synthesizing}. To verify that our method performs better when facing complex transformations in the physical world, we design physical experiments on the ImageNet classification task, and exhibit the effectiveness of our method under physical conditions.
As a byproduct, we improve the transferability of the generated adversarial examples, with a $7.5\%$ higher success rate than the state-of-the-art transfer-based black-box attack against six defense models. Specifically, when set as the initialization for black-box attacks, our method can improve the attack success rate and greatly reduce the required queries by up to 95\%.

In summary, we make several technical contributions:
\begin{itemize}
    \item We introduce an affine-invariant attack framework to generate adversarial examples with better robustness for affine transformations, and propose a kernel-based gradient estimator to greatly improve the efficiency of our algorithm;
    \item The affine-invariant adversarial examples show great transferability for defense models and can serve as a good initialization for the black-box attacks, which improves the attack success rate and greatly reduces the queries;
    \item We design physical experiments on the Imagenet classification task, in which we print all test images and introduce transformations in the physical world, and we first statistically verified that our attacks exhibit better robustness to the complex transformations in the physical world.
\end{itemize}

The remainder of this paper is organized as follows. In Section~\ref{sec:2}, we review the background of adversarial attacks and defenses. In Section~\ref{sec:3}, we explain how the proposed affine-invariant gradient estimator works to enhance the basic attacks. In Section~\ref{sec:4}, we give a detailed analysis of gradient approximation error. Furthermore, we conduct extensive experiments and have a short discussion about the relationship between affine-invariance and transferability of adversarial examples in Section~\ref{sec:5}. 
Finally, we summarize the entire paper in Section~\ref{sec:7}.

\section{Background}
\label{sec:2}

In this section, we give a detailed description of the background of adversarial attacks and defenses. Let $\boldsymbol{x}^{real}$ denote the original image; $y$ denote the ground-truth label of the corresponding $\boldsymbol{x}^{real}$; and $\boldsymbol{x}^{adv}$ denote an adversarial example for $\boldsymbol{x}^{real}$. A classifier can be denoted as $f(\boldsymbol{x}): \boldsymbol{X} \rightarrow Y$, where $\boldsymbol{x} \in \boldsymbol{X} \subset \mathbb{R}^d$ is the input image, and $Y = \{1, 2, \cdots, L\}$ is the class label with $L$ being the total number of classes. Our goal is to generate an adversarial example $\boldsymbol{x}^{adv}$, which is not visually different from the original image $\boldsymbol{x}^{real}$, but can fool the classifier. 
Therefore, we often require the $L_p$-norm of perturbation to be smaller than a threshold $\epsilon$. It is expressed as $\|\boldsymbol{x}^{adv}-\boldsymbol{x}^{real} \|_p \leq \epsilon$, where $\epsilon$ is the budget of adversarial perturbation. With $J$ denoted as the loss function, \eg{ cross entropy loss}, the goal for untargeted attacks\footnote{In this paper, we focus on the untargeted attacks. The attack methods can be easily extended to the targeted attacks.} is to maximize the loss $J(\boldsymbol{x}^{adv}, y)$, which is expressed as
\begin{equation}
    \label{eq:1}
    \mathop{\arg \max} \limits_{\boldsymbol{x}^{adv}} J(\boldsymbol{x}^{adv}, y), \ s.t. \ \|\boldsymbol{x}^{adv}-\boldsymbol{x}^{real} \|_{p} \ \leq \epsilon.
\end{equation}

Next, we introduce some typical adversarial attacks and defenses.

\subsection{White-box Attack}
\label{sec:2.1}
A white-box attack can fully access the target models. One of the most important white-box attacks is gradient-based.
FGSM~\cite{goodfellow2014explaining} is a common gradient-based attack algorithm, which proves that the linear features of deep neural networks in high-dimensional space are sufficient to generate adversarial examples. 
It performs a one-step update as
\begin{equation}
    \label{eq:0}
    \boldsymbol{x}^{adv} = \boldsymbol{x}^{real} + \epsilon \cdot \mathrm{sign}(\nabla_xJ(\boldsymbol{x}^{real}, y)), 
\end{equation}
where $\nabla_xJ(\boldsymbol{x}, y)$ is the gradient of the loss function with respect to $\boldsymbol{x}$, $\epsilon$ is the threshold of the adversarial perturbation, and $\mathrm{sign}(\cdot)$ is the sign function.
PGD~\cite{madry2017towards} extends FGSM to an iterative version. It iteratively applies gradient updates with a small step size for multiple times, and clips the adversarial examples at the end of each step as 
\begin{equation}
    \boldsymbol{x}^{adv}_{t+1} = \Pi_{\mathcal{B}_p({x}, \epsilon)} \left(\boldsymbol{x}^{adv}_{t} + \alpha \cdot  \mathrm{sign}\left(\nabla_xJ(\boldsymbol{x}^{adv}_{t}, y)\right)\right), 
\end{equation}
where  $\Pi$ is the projection operation; $\mathcal{B}_p({x}, \epsilon)$ is the $L_p$ ball centered at $\boldsymbol{x}$ with radius $\epsilon$; and $\alpha$ is the step size.

The optimization-based attacks aim to generate adversarial examples with minimum perturbation. {Deepfool}~\cite{moosavi2016deepfool} is an iterative attack method based on the idea of hyper-plane classification.
In each iteration, the algorithm adds a small perturbation to the image, gradually making the image cross the classification boundary, until the image is misclassified. The final perturbation is the accumulation of perturbations for each iteration. {Carlini \& Wagner's method (C\&W)}~\cite{carlini2017towards} is a powerful
optimization-based method. It takes a Lagrangian
form and adopts Adam~\cite{kingma2014adam} for optimization, which is written as
\begin{equation}
    \mathop{\arg \min} \limits_{\boldsymbol{x}^{adv}} \|\boldsymbol{x}^{adv}-\boldsymbol{x}^{real} \|_{p}-c\cdot J(\boldsymbol{x}^{adv}, y).
\end{equation}
C\&W is a very effective white-box attack method, but it lacks transferability to black-box models.

\subsection{Black-box Attack}
Black-box attacks cannot access the parameters and gradients of the target model, and can generally be divided into transfer-based, scored-based and decision-based attacks. 

Transfer-based attacks generate adversarial examples with a source model, then transfer it to the target model with the adversarial transferability~\cite{papernot2016practical} of the adversarial examples. 
MIM~\cite{dong2018boosting} improve the transferability by integrating a momentum term into the generation of adversarial examples. DIM~\cite{xie2019improving} proposes to improve the transferability of adversarial examples by increasing the diversity of input. It applies random resizing and padding with a given probability to the inputs at each attack iteration, and feeds the outputs to the network for the gradient calculation.
To further improve the transferability on some defense models, Dong \etal~\cite{dong2019evading} proposed Translation-Invariant Attacks (TI). This method reduces the computational complexity by convolving untranslated gradient maps with a pre-defined kernel.

Score-based attacks can only access the output scores of the target model for each input. The attacks under this setting estimate the gradient of the target model with gradient-free methods through a set of queries. NES~\cite{ilyas2018black} and SPSA~\cite{uesato2018adversarial} use sampling methods to completely approximate the true gradient. Prior-guided Random Gradient-free (P-RGF)~\cite{cheng2019improving} improves the accuracy of estimating the gradient with a transfer-based prior. $\mathcal{N}$ATTACK\cite{li2019nattack} learns a probability density distribution centered around the input, and samples from the distribution to generate adversarial examples. 

Decision-based attacks are more challenging since the attacker can only acquire the discrete hard-label predictions of the target model. Decision-based attacks such as Boundry~\cite{brendel2017decision} attack and Evolutionary~\cite{dong2019efficient} attack also play an important role in black-box attacks. 

\subsection{Defense Methods}
A large variety of adversarial defense methods have been proposed to resist the increasing threat of adversarial attacks. One of the important ways is to transform the input before feeding it to the network, to reduce the influence of the adversarial perturbation; such methods include JPEG Compression~\cite{dziugaite2016study}, Bit-depth Reduction~\cite{xu2017feature}, and denoising methods with auto-encoder or other generative models~\cite{liao2018defense, samangouei2018defense}. Randomization-based defenses introduce randomness to the networks to mitigate the effect of adversarial perturbation. Previous works mostly added randomness to the input~\cite{xie2017mitigating} or the model~\cite{engstrom2019exploring}. Adversarial training~\cite{madry2017towards, tramer2017ensemble, kannan2018adversarial, zhang2019theoretically} is another popular defense method, which expands adversarial examples into training data to make the networks more robust against the adversarial perturbation. Certified defenses~\cite{raghunathan2018certified, zhang2019defending} provide a certificate that guarantees the robustness of defense models under some threat models, and play an increasingly important role in defense methods.

\section{Methodology}
\label{sec:3}

In this section, we give a detailed description of our proposed affine-invariant gradient estimator. In Sec.~\ref{sec:3.1}, we formulate our method as maximizing the expectation of adversarial loss for affine transformation, which is decomposed into translation and scaling-rotation transformations. In Sec.~\ref{sec:3.2}, we show how to estimate the gradient of the loss function in the convolution form. In Sec.~\ref{sec:3.3}, we formulate the solution of kernel matrices in our estimator. In Sec.~\ref{sec:3.4}, we show the attack algorithms of our method.

\subsection{Problem Formulation}
\label{sec:3.1}

In order to generate more robust adversarial examples, we propose an affine-invariant attack method, which optimizes the $x^{adv}$ to maximize the expectation of adversarial loss in the preset affine transformation space domain as
\begin{gather}
\label{eq:2}
    \mathop{\arg \max} \limits_{\boldsymbol{x}^{adv}} \mathbb{E}_{a \sim \mathbb{A}}[J(\mathcal{F}_a(\boldsymbol{x}^{adv}), y)],\;
   s.t. \ \|\boldsymbol{x}^{adv}-\boldsymbol{x}^{real} \|_{\infty} \ \leq \epsilon,
\end{gather}
where $a$ is the random variable to affine transformation; $\mathbb{A}$ is the probability distribution of $a$; and $\mathcal{F}_a(\cdot)$ is the transformation function of $a$, which returns the transformed image.

Considering subtle camera movement in a long distance, we can approximately decompose an affine transformation as translation, rotation, and uniform scaling transformations, while ignoring shear and flip in our method. Therefore, for any 2-D affine transformation matrix $\boldsymbol{M}_a$, we have:

\begin{equation}
\label{eq:3}
\begin{split}
    \boldsymbol{M}_a
    &= \begin{bmatrix}
    1 & 0 & m      \\
    0 & 1 & n \\
    0 & 0 & 1
    \end{bmatrix}
    \begin{bmatrix}
    s\cdot \cos\theta & -s\cdot \sin\theta & 0      \\
    s\cdot \sin\theta & s\cdot \cos\theta & 0 \\
    0 & 0 & 1
    \end{bmatrix} \\
    &= \boldsymbol{M}_t \cdot \boldsymbol{M}_q,
\end{split}
\end{equation}
where $\theta$ is the rotation angle; $s$ is the scaling factor; $m$ is the translation length in the x-axis, $n$ is the translation length in the y-axis; $t$ and $q$ are the random variables of the decomposed translation and scaling-rotation transformations; and $\boldsymbol{M}_t$ and $\boldsymbol{M}_q$ are the transformation matrices of $t$ and $q$. According to Eq.~\eqref{eq:3}, the affine transformation function $\mathcal{F}_a(\boldsymbol{x})$ should be a composition of both the translation function and the scaling-rotation function, which means:
\begin{equation}
\begin{split}
    \mathcal{F}_a(\boldsymbol{x})=\mathcal{F}_t(\mathcal{F}_q(\boldsymbol{x}))=\mathcal{F}_{t,q}(\boldsymbol{x}),
\end{split}
\end{equation}
where $\mathcal{F}_t$ is the translation function of $t$; $\mathcal{F}_q$ is the scaling-rotation function of $q$;  and $\mathcal{F}_{t,q}$ is the composition function of $\mathcal{F}_t$ and $\mathcal{F}_q$. The decomposition process is shown in Fig.~\ref{fig:2}.

\begin{figure}[tp]
  \centering
  \includegraphics[width=\columnwidth]{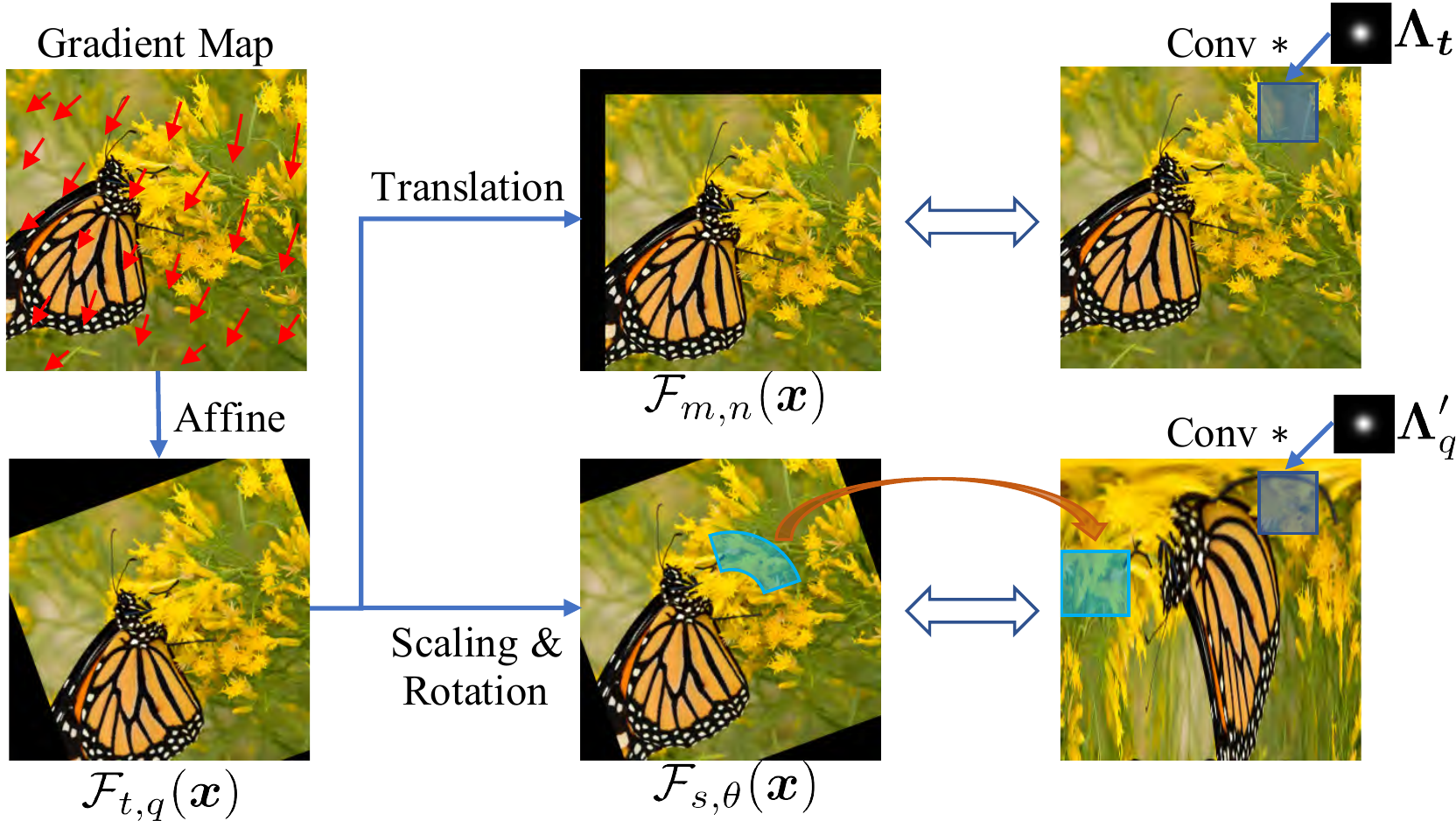} 
  \caption{In order to obtain an affine invariance property, we need to calculate the expectation of the gradient map obtained under different affine transformations. An affine transformation can be decomposed as translation and scaling-rotation transformations. A translation calculation can be accelerated by a convolution operation referring to ~\cite{dong2019evading}. A scaling-rotation calculation is equivalent to convolution in corresponding polar space.} 
  \vspace{-1ex}
  \label{fig:2} 
\end{figure}

In this way, the affine transformation in the optimization problem is decomposed into two simple transformations as:
\begin{equation}
    \mathop{\arg \max} \limits_{\boldsymbol{x}^{adv}} \mathbb{E}[ J(\mathcal{F}_{t,q}(\boldsymbol{x}^{adv}), y)], \;
  s.t. \ \|\boldsymbol{x}^{adv}-\boldsymbol{x}^{real} \|_{p} \ \leq \epsilon.
\label{optimal}
\end{equation}
In order to obtain the optimal solution, we need to calculate the expectation of the gradient in Eq.~\eqref{optimal}.

\subsection{Gradient Calculation}
\label{sec:3.2}
In this section, we provide a detailed calculation of the objective gradient, which is the core of gradient-based attacks. The gradient of weighted loss to the input image $\boldsymbol{x}$ is expressed as:
\begin{equation}
\label{eq:loss}
\begin{split}
    G_{\hat{\boldsymbol{x}}} &= \nabla_{\boldsymbol{x}} \mathbb{E}[J(\mathcal{F}_{t,q}(\boldsymbol{x}), y)] \Big|_{\boldsymbol{x}=\hat{\boldsymbol{x}}} \\
    &= \mathbb{E}[\nabla_{\mathcal{F}_{t,q}(\boldsymbol{x})} J(\mathcal{F}_{t,q}(\boldsymbol{x}), y) \frac{\partial \mathcal{F}_{t,q}(\boldsymbol{x})}{\partial \boldsymbol{x}}] \Big|_{\boldsymbol{x}=\hat{\boldsymbol{x}}} \\
    &= \mathbb{E}[\mathcal{F}_{t,q}^{-1} (\nabla_{\boldsymbol{x}} J(\boldsymbol{x}, y) \Big|_{\boldsymbol{x}=\mathcal{F}_{t,q}(\hat{\boldsymbol{x}})} )],
\end{split}
\end{equation}
where $\mathcal{F}_{t,q}(\boldsymbol{x})$ is replaced with $\boldsymbol{x}$ in the final step.

In order to analyze the gradient in terms of the original images, we introduce two assumptions. The first one is that $J(\boldsymbol{x},y)$ satisfies the gradient Lipschitz condition, which means smoothness of the gradient function $\nabla_{\boldsymbol{x}} J(\boldsymbol{x},y)$. The second one is that the weighted sum of the distance between the transformed images and the original one is upper-bounded. With these two assumptions, we can approximate the gradient of the transformed image, by basing that approximation on the gradient of the original. We provide a detailed analysis of the gradient approximation error in Sec.~\ref{sec:4}. Accordingly, $G_{\hat{\boldsymbol{x}}}$ is simplified as:
\begin{equation}
\label{eq:simplified}
    G_{\hat{\boldsymbol{x}}} \approx \mathbb{E}[\mathcal{F}_{t,q}^{-1} (\nabla_{\boldsymbol{x}} J(\boldsymbol{x}, y) \Big|_{\boldsymbol{x}=\hat{\boldsymbol{x}}} )] \triangleq \hat{G}_{\hat{\boldsymbol{x}}},
\end{equation}
where the approximated expectation is denoted as $\hat{G}_{\hat{\boldsymbol{x}}}$.

For convenience in the following steps, we assume that $q$ and $t$ are two independent random variables, which can be represented with detailed parameters $(s,\theta)$ and $(m,n)$ separately. Then, the gradient is expressed as:
\begin{equation}
\label{eq:8}
\begin{split}
    \hat{G}_{\hat{\boldsymbol{x}}}
    &= \mathbb{E}[\mathcal{F}_{q}^{-1}(\mathcal{F}_{t}^{-1}(\nabla_{\boldsymbol{x}} J(\boldsymbol{x},y) \Big|_{\boldsymbol{x}=\hat{\boldsymbol{x}}}) )] \\
    &= \mathbb{E}_{(s,\theta) \sim \mathbb{Q}}[\mathcal{F}_{s,\theta}^{-1} (\mathbb{E}_{(m,n) \sim \mathbb{T}}(\mathcal{F}_{m,n}^{-1}(\mathcal{G}_{\hat{\boldsymbol{x}}})))] \\
    &\triangleq \mathcal{Q} \comp \mathcal{T} (\mathcal{G}_{\hat{\boldsymbol{x}}})),
\end{split}
\end{equation}
where $q$ is split into $(s,\theta)$, and $t$ is split into $(m,n)$. In the final equation, we denote $\nabla_{\boldsymbol{x}} J(\boldsymbol{x},y) |_{\boldsymbol{x}=\hat{\boldsymbol{x}}}$ as $\mathcal{G}_{\hat{\boldsymbol{x}}}$, and the two expectations functions as $\mathcal{Q}$ and $\mathcal{T}$.

In actual implementation, sampling a series of transformed images for gradient calculation is a feasible but inefficient method~\cite{athalye2018synthesizing}. In our method, we discretize the four random variables $(s,\theta,m,n) = \{(s_i, \theta_i, m_i, n_i) \mid i \in \mathbb{Z}\}$ to simplify the calculation of Eq.~\eqref{eq:8}. The overall framework of affine transformation decomposition and equivalence of the decomposed transformations is shown in Fig.~\ref{fig:2}. In the following part, we provide the equivalent convolution forms.

\subsubsection {\bf {Equivalence of Translation}}
Since images are discrete 2D grids, to simplify the calculation we discretize the translation into pixel-wise shifts. Furthermore, we can set the value of $m$ and $n$ to be the moving step size in basic directions. Referring to~\cite{dong2019evading}, the translation part can be equivalent to convolving the gradient with a kernel composed of all weights as:
\begin{equation}
\label{eq:9}
\begin{split}
    \mathcal{T}(\boldsymbol{x})
    = \sum\limits_{m,n}p(m,n) \mathcal{F}_{{-}m, {-}n}(\boldsymbol{x}) 
    \Leftrightarrow \boldsymbol{\Lambda_t} * \boldsymbol{x},
\end{split}
\end{equation}
where $m$ and $n$ also represent the shifting step size in two basic directions; $p(m,n)$ is the probability function of the translation transformation; and $\boldsymbol{\Lambda}_t$ is the equivalent  translation kernel matrix. Taking a finite number of translation transformations, \ie, $m \in \{-k_1, \cdots, 0, \cdots, k_1\}$ and $n \in \{-k_2, \cdots, 0, \cdots, k_2\}$, the size of $\boldsymbol{\Lambda}_t$ is  $(2k_1+1)\times(2k_2+1)$, with $\boldsymbol{\Lambda}_{t_{m, n}} = p(-m,-n)$. 

\subsubsection{\bf {Equivalence of Rotation and Scaling}}
For rotation and scaling, normal convolution operation does not work to simplify the calculations. However, scaling can be linearized and approximated as radial shifts in polar space, when it comes to subtle transformation. We can project the original gradient image into polar space, as shown in Fig.~\ref{fig:2}. Then, rotation and scaling can be approximated as translation in the polar space. We replace the scaling factor $s$ with a radial shift distance $r$, such that $\mathcal{Q}(x)$ is expressed as:
\begin{equation}
\label{eq:11}
\begin{split}
    \mathcal{Q}(\boldsymbol{x}) 
    &\approx \mathcal{P}^{-1} \sum_{i,j} \mathcal{P} (p(r_i,\theta_j) \mathcal{F}_{r_i,\theta_j}^{-1}(\boldsymbol{x})) \\
    &\Leftrightarrow \mathcal{P}^{-1} \sum_{i,j} p'(u_i,v_j) \mathcal{F}_{u_i,v_j}^{-1}(\mathcal{P}(\boldsymbol{x})),
\end{split}
\end{equation}
where $\mathcal{P}(\cdot)$ and $\mathcal{P}^{-1}(\cdot)$ are the polar transformation and inverse polar transformation, while $u$ and $v$ are the corresponding random variables in polar space. 
Similarly, by discretizing the translation in polar space into pixel-wise shifts, we can get:
\begin{equation}
     \mathcal{Q}(\boldsymbol{x}) \Leftrightarrow \mathcal{P}^{-1} \left(\boldsymbol{\Lambda'_q} * \mathcal{P}(\boldsymbol{x}) \right), 
\end{equation}
where $\boldsymbol{\Lambda'_q}$ is the translation kernel matrix of size $(2l_1+1)\times(2l_2+1)$ in polar space.

Finally, the total gradient calculation is equivalent to some simple operations such that:
\begin{equation}
    G_{\hat{\boldsymbol{x}}} \approx \hat{G}_{\hat{\boldsymbol{x}}} \Leftrightarrow \mathcal{P}^{-1}(\boldsymbol{\Lambda}_q' * \mathcal{P}(\boldsymbol{\Lambda}_t * \mathcal{G}_{\hat{x}})),
\end{equation}
where $\mathcal{G}_{\hat{x}}$ is first convolved with a translation kernel $\boldsymbol{\Lambda}_t$, then convolved with another translation kernel $\boldsymbol{\Lambda}_q'$ in polar space.

\subsection{Kernel Matrix}
\label{sec:3.3}
For the translation part, we set the translation step size in a limited range, with $m \in \{-k_1, \cdots, 0, \cdots, k_1\}$ and $n \in \{-k_2, \cdots, 0, \cdots, k_2\}$. To ensure the attack performance on the untransformed images, we follow the basic principle that the more the input image changes, the lower the weight that should be assigned. Therefore, we empirically set matrix $\boldsymbol{\Tilde{\Lambda}}_t$ to follow Gaussian distribution, \ie $\boldsymbol{\Tilde{\Lambda}}_{t_{m,n}} = (2 \pi \sigma_{t_1}\sigma_{t_2})^{-1} {\exp}\{-(m^2+n^2)(2\sigma_{t_1}\sigma_{t_2})^{-1}\}$, where $\sigma_{t_1}=k_1/\sqrt{3}$ and $\sigma_{t_2}=k_2/\sqrt{3}$.
To ensure that the convolved gradients are at the same level as the original ones, we need to set $\boldsymbol{\Lambda}_t$ as a normalized matrix:
\begin{equation}
\label{eq:14}
    \boldsymbol{\Lambda}_t = \boldsymbol{\Tilde{\Lambda}}_t/ \| \boldsymbol{\Tilde{\Lambda}}_t \|_1.
\end{equation}

As for the rotation and scaling parts, we can define the kernel matrix directly in polar space just as with the translation part. Following the same principle, we set $\boldsymbol{\Tilde{\Lambda}}'_{q_{i,j}} = (2 \pi \sigma_{u} \sigma_{v})^{-1} {\exp}\{-(i^2+j^2)(2\sigma_{u} \sigma_{v})^{-1}\}$, where $i \in \{-l_1, \cdots, 0, \cdots, l_1\}$, $j \in \{-l_2, \cdots, 0, \cdots, l_2\}$ and $\sigma_{u} = l_1/\sqrt{3}$, $\sigma_{v} = l_2/\sqrt{3}$. Then the kernel is defined as:
\begin{equation}
\label{eq:15}
    \boldsymbol{\Lambda}'_q = \boldsymbol{\Tilde{\Lambda}}'_q/ \| \boldsymbol{\Tilde{\Lambda}}'_q \|_1.
\end{equation}
The normalized convolution kernels $\boldsymbol{\Lambda}'_q$ in Eq.~\eqref{eq:14} and $\boldsymbol{\Lambda}_t$ in Eq.~\eqref{eq:15} are used to enhance the affine invariance of the gradient map.

\subsection{Attack Algorithms}
\label{sec:3.4}
\begin{algorithm}[tp]
\caption{AI-PGD}
\label{alg:A}
\begin{algorithmic}[1]
\renewcommand{\algorithmicrequire}{\textbf{Input:}}
\renewcommand{\algorithmicensure}{\textbf{Output:}}
\REQUIRE A classifier $f$ with $J$ as its loss function; a natural image $\boldsymbol{x}$ and its true label $y$.
\REQUIRE The size of perturbation $\epsilon$; total iterations $T$; step size $\alpha$; two predefined convolutional kernel $\boldsymbol{\Lambda}_t$ and $\boldsymbol{\Lambda}'_q$.
\ENSURE The corresponding affine-invariant adversarial example $\boldsymbol{x}^{adv}$.
\STATE {Generate random initial noise $\boldsymbol{x}^{init}$}
\STATE {$\boldsymbol{x}^{adv}_{0} \gets \boldsymbol{x}+\boldsymbol{x}^{init}$}
\FOR{$t=0$ to $T-1$}
\STATE {Feed $\boldsymbol{x}^{adv}_{t}$ to $f$ and calculate the corresponding gradient
\begin{center}
$\mathcal{G}_{\boldsymbol{x}^{adv}_{t}} \gets \nabla_{\boldsymbol{x}} J(\boldsymbol{x}, y) \Big|_{\boldsymbol{x}=\boldsymbol{x}^{adv}_{t}}$
\end{center}
}
\STATE {Enhance the gradient with the affine-invariant gradient estimator as:
\begin{center}
$G_{\boldsymbol{x}^{adv}_{t}} \gets \mathcal{P}^{-1}(\boldsymbol{\Lambda}_q' * \mathcal{P}(\boldsymbol{\Lambda}_t * \mathcal{G}_{\boldsymbol{x}^{adv}_{t}}))$
\end{center}
}
\STATE {Update $\boldsymbol{x}^{adv}_{t+1}$ by applying the sign of estimated gradient and projection operation as:\\
\begin{center}
    $\boldsymbol{x}^{adv}_{t+1} \gets \Pi_{\mathcal{B}_p({x}, \epsilon)} \left(\boldsymbol{x}^{adv}_{t} + \alpha \cdot \  \mathrm{sign}(G_{\boldsymbol{x}^{adv}_{t}})\right)$
\end{center}
}
\ENDFOR
\RETURN $\boldsymbol{x}^{adv} \gets \boldsymbol{x}^{adv}_{T}$
\end{algorithmic}
\end{algorithm}

In Sec.~\ref{sec:3.2} and Sec.~\ref{sec:3.3}, we show how to calculate the gradient and corresponding kernel matrix. Here, we introduce the updating strategy of our attack method for generating adversarial examples. Essentially, our method is also related to the gradient, so it can be easily integrated into other gradient-based attack methods introduced in Sec.~\ref{sec:2}, such as FGSM~\cite{goodfellow2014explaining}, PGD~\cite{madry2017towards}, \etc. For gradient-based attack methods such as PGD, we need to calculate the gradient $\nabla_xJ(\boldsymbol{x}^{adv}_t, y)$ of the current solution $\boldsymbol{x}^{adv}_t$ in each step. In our method, however, we just need to replace the normal gradient with the result of $G_{\hat{\boldsymbol{x}}}$ obtained by our proposed affine-invariant gradient estimator in Sec.~\ref{sec:3.2}. 

For example, when combined with one-step methods such as FGSM~\cite{goodfellow2014explaining} (AI-FGSM), the updating strategy is written as:
\begin{equation}
\label{eq:111}
    \boldsymbol{x}^{adv} = \boldsymbol{x}^{real} +\epsilon \cdot \mathrm{sign}(G_{\boldsymbol{x}^{real}}).
\end{equation}
When combined with the iterative methods such as PGD~\cite{madry2017towards} (AI-PGD), the updating strategy is written as:
\begin{equation}
\label{eq:12}
    \boldsymbol{x}^{adv}_{t+1} =  \Pi_{\mathcal{B}_p({x}, \epsilon)} \left(\boldsymbol{x}^{adv}_{t} + \alpha \cdot \  \mathrm{sign}(G_{\boldsymbol{x}^{adv}_{t}})\right).
\end{equation}
The detailed algorithm of AI-PGD is summarized in Algorithm~\ref{alg:A}.
Our method can be similarly integrated into other gradient-based attack methods such as MIM~\cite{dong2018boosting} and DIM~\cite{xie2019improving} as AI-MIM and AI-DIM.

\section{Analysis of Gradient Approximation Error}
\label{sec:4}
In Sec.~\ref{sec:3.2}, we introduce a gradient approximation to simplify the gradient calculation in Eq.~\eqref{eq:simplified} as:
\begin{equation}
\label{eq:4.1}
    \mathbb{E}[ \mathcal{F}_{t,q}^{-1} (\nabla_{\boldsymbol{x}} J(\boldsymbol{x}, y)  \Big|_{\boldsymbol{x}=\mathcal{F}_{t,q}(\hat{\boldsymbol{x}})} )]
    \approx  \mathbb{E}[ \mathcal{F}_{t,q}^{-1} (\nabla_{\boldsymbol{x}} J(\boldsymbol{x}, y) \Big|_{\boldsymbol{x}=\hat{\boldsymbol{x}}} )].
\end{equation}
In this section, we give a detailed analysis of the gradient approximation error to show the rationality of using gradient approximation in our method. 

We let $\boldsymbol{g}_1 = \mathbb{E}[ \mathcal{F}_{t,q}^{-1} (\nabla_{\boldsymbol{x}} J(\boldsymbol{x}, y) \Big|_{\boldsymbol{x}=\hat{\boldsymbol{x}}} )]$ and $\boldsymbol{g}_2  = \mathbb{E}[ \mathcal{F}_{t,q}^{-1} (\nabla_{\boldsymbol{x}} J(\boldsymbol{x}, y) \Big|_{\boldsymbol{x}=\mathcal{F}_{t,q}(\hat{\boldsymbol{x}})} )]$ for convenience. Furthermore, we introduce the two main assumptions used in our analysis.

DNNs introduce ReLU activation function into the structure, so the loss function $J(\boldsymbol{x}, y)$ of the neural networks does not satisfy the Lipschitz condition. However, recent studies~\cite{allen2019convergence,gao2019convergence} have shown that the optimization landscape is almost-convex and semi-smooth with over-parameterized neural networks, showing the semi-smoothness of the loss function. Therefore, we can strengthen this conclusion into the first assumption.

\begin{assumption}
    The loss function $J(\boldsymbol{x}, y)$ satisfies the gradient Lipschitz condition
        , which means smoothness of the gradient function $\nabla_x J(\boldsymbol{x}, y)$. A subtle affine transformation will not affect its smoothness, which is expressed as:
    \begin{gather}
        \left \|\mathcal{F}_{t,q}^{-1}(\nabla_x J(\boldsymbol{x}_2, y)) - \mathcal{F}_{t,q}^{-1}(\nabla_x J(\boldsymbol{x}_1, y)) \right\|_2 \\ \nonumber
        \leqslant c_1 \left\|\boldsymbol{x}_2 -\boldsymbol{x}_1 \right\|_2, 
    \end{gather}
    where $c_1$ is a positive constant.
\end{assumption}

For natural images, the values of two adjacent pixels are usually continuous and gradual. Therefore, when the affine transformation we perform on the image is smaller, the Euclidean distance between the transformed image and the original image is also smaller. Now we can lay out our second assumption.
\begin{assumption}
    The expectation of the distance between the affine transformed image and the original one is upper-bounded as:
    \begin{equation}
        \mathbb{E}[  \left \|\mathcal{F}_{t,q} (\boldsymbol{x}) - \boldsymbol{x} \right \|_2 ]\leqslant c_2,
    \end{equation}
    where $c_2$ is a positive constant.
\end{assumption}

With the two assumptions, we can get Corollary~\ref{corollary:1} that the Euclidean distance between $\boldsymbol{g}_1$ and $\boldsymbol{g}_2$ is upper-bounded. The detailed proof is shown in Appendix~A.
\begin{corollary}
\label{corollary:1}
    The Euclidean distance between $\boldsymbol{g}_1$ and $\boldsymbol{g}_2$ is upper-bounded as:
    \begin{equation}
        \left\|\boldsymbol{g}_2-\boldsymbol{g}_1 \right\|_2 \leqslant c_1 \cdot c_2.
    \end{equation}
\end{corollary}

Now we get an upper bound of the approximation error. Since the gradients used for adversarial example generation in Sec.~\ref{sec:3} would be normalized, we also care about the directions between $\boldsymbol{g}_1$ and $\boldsymbol{g}_2$ except distance. Then we analyze the cosine similarity between them with another assumption. 
\begin{assumption}
The norms of $\boldsymbol{g}_1$ and $\boldsymbol{g}_2$ are larger than a positive constant $c_3$ as
\begin{equation}
    \label{eq:4.6}
    \left \| \boldsymbol{g}_1 \right \|_2 \geqslant c_3;\ \ \ \ \  \left \| \boldsymbol{g}_2 \right \|_2 \geqslant c_3.
\end{equation}
\end{assumption}
Assumption~3 can be satisfied when the model does not cause gradient vanishing, otherwise the adversarial examples cannot be generated since the gradients are zero. We then can analyze the difference of directions between $\boldsymbol{g}_1$ and $\boldsymbol{g}_2$ by their cosine similarity. We finally give the following corollary. The detailed proof is shown in Appendix~A.
\begin{corollary}
\label{corollary:2}
    The cosine similarity of $\boldsymbol{g}_1$ and $\boldsymbol{g}_2$ is lower-bounded as: 
    \begin{equation}
        cossim(\boldsymbol{g}_1, \boldsymbol{g}_2) \geqslant 1 - \frac{(c_1 c_2)^2}{2c_3^2},
    \end{equation}
    where $cossim(\cdot, \cdot)$ is the cosine similarity function.
\end{corollary}

In our method, we use Gaussian kernels, which can cause a relatively small $c_2$. With Corollary~\ref{corollary:1} and Corollary~\ref{corollary:2}, we can approximate the gradient calculation as Eq.~\eqref{eq:4.1} with a small approximation error.

\renewcommand{\arraystretch}{1.0} 
\begin{table*}[tp]  
  
  \centering  
  \caption{The ASRs (\%) of adversarial attacks under different rotations. The adversarial examples are crafted for Inc-v3 using FGSM, PGD, MIM, DIM and their extensions AI-FGSM, AI-PGD, AI-MIM, AI-DIM. We set the scaling factor and translation offset to be $s=0,m=0,n=0$. We test the performance on three models---Inc-v3, Inc-v4 and Ens-AT.}
  \fontsize{8.5}{10}\selectfont  

  \setlength{\tabcolsep}{0.9mm}{
    \begin{tabular}{c|c|c|c|c|c|c|c|c|c|c|c|c|c|c|c}  
    \hline  
    \multirow{2}*{Method}&  
    \multicolumn{3}{c|}{$\theta=-30 ^{\comp}$}&\multicolumn{3}{c|}{ $\theta=-15^{\comp}$}&\multicolumn{3}{c|}{$\theta=0 ^{\comp}$}&\multicolumn{3}{c|}{ $\theta=15 ^{\comp}$}&\multicolumn{3}{c}{ $\theta=30 ^{\comp}$}\cr\cline{2-16}  
    &Inc-v3&Inc-v4&Ens-AT&Inc-v3&Inc-v4&Ens-AT&Inc-v3&Inc-v4&Ens-AT&Inc-v3&Inc-v4&Ens-AT&Inc-v3&Inc-v4&Ens-AT\cr  
    \hline  
    \hline  
    FGSM&72.6&65.8&38.6 &71.3&57.3&34.7 &{\bf 79.9}&35.6&12.1 &70.0&57.2&34.9 &71.3&65.7&42.0 \cr  
    AI-FGSM&{\bf 83.0}&{\bf 81.0}&{\bf73.1} &{\bf 78.0}&{\bf 73.4}&{\bf59.7} &63.2&{\bf 44.0}&{\bf23.6} &{\bf 77.5}&{\bf 72.9}&{\bf58.9} &{\bf 79.9}&{\bf 80.7}& {\bf63.9}\cr  
    \hline
    PGD&60.0&37.1& 20.4&75.2&30.8& 14.2&{\bf100}&24.8&6.3 &73.1&31.5&15.0 &56.8&36.8&17.6 \cr  
    AI-PGD&{\bf 85.2}&{\bf 76.3}&{\bf 58.1 }&{\bf 90.3}&{\bf 66.3}&{\bf49.0 }&\bf100&{\bf 40.9}&{\bf 18.9} &{\bf 88.7}&{\bf 66.8}&{\bf49.9 }&{\bf 84.6}&{\bf 74.3}&{\bf54.6} \cr  
    \hline
    MIM&84.2&69.8&48.2 &92.4&63.7&36.7 &{\bf100}&50.5&16.3 &91.1&65.4&38.0 &83.1&69.4&43.1 \cr    
    AI-MIM&{\bf 92.4}&{\bf 85.8}& {\bf70.5} &{\bf 96.0}&{\bf 82.4}& {\bf64.3} &{\bf100}&{\bf 59.0}& {\bf33.3}&{\bf 95.7}&{\bf 80.8}&{\bf65.0} &{\bf 90.6}&{\bf 86.6}&{\bf70.8} \cr  
    \hline
    DIM&{\bf 94.9}&{ 82.8}&{50.3  }&{ 97.3}&{ 79.6}&{46.0 }&\bf100&{\bf75.4}&21.8 &97.7&78.8&46.8 &94.6&83.1&47.5 \cr  
    AI-DIM&{\bf 94.9}&{\bf 88.8}&{\bf 76.9 }&{\bf 97.5}&{\bf 86.0}&{\bf 71.8 }&{ 99.9}&72.6&\bf 39.2 &{\bf 98.0}&{\bf 86.7}& \bf 70.2 &{\bf 95.3}&{\bf 90.6}&\bf 74.6 \cr 
    \hline  
    EOT & 82.1 & 52.0 & 10.6 & 82.2 & 50.7 & 9.1 & 82.3 & 50.3 & 10.2 & 82.5 & 51.4 & 10.3 & 82.4 & 51.6 & 9.6\cr
    \hline  
    \end{tabular}  
    }
    \vspace{-1ex}
    \label{tab:rotation}  
\end{table*} 

\renewcommand{\arraystretch}{1.0} 
\begin{table*}[tp]  
  
  \centering  
  \caption{The ASRs (\%) of adversarial attacks under different scalings. The adversarial examples are crafted for Inc-v3 using FGSM, PGD, MIM, DIM, AI-FGSM, AI-PGD, AI-MIM, and AI-DIM, respectively. We set the rotation angle and translation offset to be $\theta=30^{\comp},m=20,n=20$. We test the performance on Inc-v3, Inc-v4 and Ens-AT.}
  \fontsize{8.5}{10}\selectfont  
  \setlength{\tabcolsep}{0.9mm}{
    \begin{tabular}{c|c|c|c|c|c|c|c|c|c|c|c|c|c|c|c}
    \hline
    \multirow{2}*{Method}&  
    \multicolumn{3}{c|}{$s=0.5$}&\multicolumn{3}{c|}{ $s=0.7$}&\multicolumn{3}{c|}{$s=1.0$}&\multicolumn{3}{c|}{ $s=1.3$}&\multicolumn{3}{c}{ $s=1.5$}\cr\cline{2-16}  
    &Inc-v3&Inc-v4&Ens-AT&Inc-v3&Inc-v4&Ens-AT&Inc-v3&Inc-v4&Ens-AT&Inc-v3&Inc-v4&Ens-AT&Inc-v3&Inc-v4&Ens-AT\cr
    \hline
    \hline
    FGSM&72.0&66.4&55.2 &75.4 &70.9 &43.1 &{71.7 }&66.3 &45.8 &59.8 &51.6 &37.9 &56.3 &51.7 &41.5 \cr
    AI-FGSM&{\bf 87.3}&{\bf 87.7}&{\bf 64.2} &{\bf 87.0}&{\bf 87.6}&{\bf 63.6} &\bf 81.0 &{\bf 80.6}&{\bf 65.9} &{\bf 69.2}&{\bf 68.9}&{\bf 56.1} &{\bf 64.6}&{\bf 59.3}& {\bf 52.7}\cr
    \hline
    PGD&38.7&28.1 & 26.9&44.0 &34.3 &22.4 &{ 58.4}& 36.7&26.4 &50.6 &31.1&20.7  &41.7 &26.5 &23.1 \cr
    AI-PGD&{\bf 82.0}&{\bf 83.4}&{\bf 62.7}&{\bf 85.8}&{\bf 82.9}&{\bf 58.8}&\bf 85.1&{\bf 73.6}&{\bf 57.1} &{\bf 73.6}&{\bf 59.6}&{\bf 48.1}&{\bf 61.5}&{\bf 51.6}&{\bf 44.2} \cr
    \hline
    MIM& 66.0 & 60.0 &44.8 &75.6 &72.0 &40.3 &{83.5 }&69.4 &46.8 &74.3 &59.8 &42.6 &67.9 &56.9&44.6  \cr
    AI-MIM&{\bf 86.7}&{\bf 90.2}& {\bf 67.2} &{\bf 91.7}&{\bf 91.2}& {\bf 68.7} &{\bf 90.7}&{\bf 86.3}& {\bf 72.1}&{\bf 83.4}&{\bf 75.5}&{\bf 64.7} &{\bf 75.7}&{\bf 68.4}&{\bf 61.1} \cr
    \hline
    DIM&{70.7}&{ 64.7}&{44.8 }&{83.2 }&{78.1}&{44.1}&94.7&{82.9}&51.2 &\bf92.1&76.6 &54.4 &85.5 &70.3 &57.7  \cr  
    AI-DIM&{\bf 90.0}&{\bf 91.1}&{\bf 67.2}&{\bf 93.0}&{\bf 93.1}&{\bf 70.9}&{\bf 95.6}&\bf90.7&\bf75.2 &{ 90.6}&{\bf 82.6}& \bf 71.3 &{\bf 85.5}&{\bf 77.9}&\bf 69.1 \cr 
    \hline  
    EOT & 52.6 & 36.2 & 16.8 & 74.6 & 47.0 & 14.4 & 82.4 & 51.6 & 9.6 & 83.6 & 62.1 & 24.9 & 81.9 & 61.4 & 30.9\cr
    \hline  
    \end{tabular}  
    }
    \label{tab:scaling}  
    \vspace{-2ex}
\end{table*} 

\section{Experiments}
\label{sec:5}
In this section, we introduce our experiments and prove the effectiveness of our method. In Sec.~\ref{sec:5.1}, we introduce the experimental settings. We then test the affine invariance and efficiency of our methods compared to some basic attacks and EOT\cite{athalye2018synthesizing} in Sec.~\ref{sec:5.2}. In Sec.~\ref{sec:5.2.4}, we further verify the robustness of our method to more complex transformations in the physical world. In Sec.~\ref{sec:5.3}, we verify the robustness of our methods to defense models under black-box settings. Next, we provide an ablation study for our methods in Sec.~\ref{sec:5.4}. Finally, we include a short discussion about the affine-invariance and transferability of adversarial examples in Sec.~\ref{sec:5.5}.

\subsection{Experimental Settings}
\label{sec:5.1}


We first design experiments to show the improvement of our proposed attacks on affine invariance in the digital world, and then further introduce physical conditions to verify the robustness of our method to affine transformations. Finally, we demonstrate that our approach can also improve the transferability of adversarial attacks to defense models. Below are some details of the experimental setup.

\textbf{Dataset and Models.}
We use an ImageNet-compatible dataset\footnote{\url{https://github.com/cleverhans-lab/cleverhans/tree/master/cleverhans_v3.1.0/examples/nips17_adversarial_competition/dataset}} comprised of the 1,000 images that were used in the NeurIPS 2017 adversarial competition. For models, we choose four naturally trained models and six defense models according to the RealSafe platform~\cite{dong2020benchmarking}. These models are naturally trained Inception v3 (Inc-v3)~\cite{szegedy2016rethinking}; Inception v4 (Inc-v4)~\cite{szegedy2016inception}; Inception ResNet v2 (IncRes-v2)~\cite{szegedy2016inception} and ResNet v2-152 (Res-v2-152)~\cite{he2016identity}; Ensemble Adversarial Training (Ens-AT)~\cite{tramer2017ensemble}; Adversarial Logit Pairing (ALP)~\cite{kannan2018adversarial}; JPEG Compression~\cite{dziugaite2016study}; Bit-depth Reduction (Bit-Red)~\cite{xu2017feature}; Random Resizing and Padding (R\&P)~\cite{xie2017mitigating}; and RandMix~\cite{zhang2019defending}. Furthermore, we use Inc-v3 as the backbone model for defenses based on input transformations such as JPEG and Bit-Red.



\textbf{Evaluation Metrics.} We use the attack success rate as the evaluation metrics referring to RealSafe~\cite{dong2020benchmarking}. The attack success rate of an untargeted attack on the classifier $f$ is defined as:
\begin{equation}
    \mathrm{ASR}(\mathcal{A}_{\epsilon,p}, f)=\frac{1}{M} \sum \limits _{i=1}^N \mathbf{1} (f(\boldsymbol{x_i'}) = y_i \wedge f(\mathcal{A}_{\epsilon,p}(\boldsymbol{x_i'})) \neq y_i),
\end{equation}
where $\{ \boldsymbol{x_i'}, y_i\}^N_{i=1}$ is the test set; $x_i' = \mathcal{F}_{a}(x_i)$; $\mathbf{1}(\cdot)$ is the indicator function; $\mathcal{A}_{\epsilon,p}$ means the attack method that generates the adversarial examples with perturbation budget $\epsilon$ under the $L_p$ norm; and $M=\sum _{i=1}^N \mathbf{1}(f(\boldsymbol{x_i'}) = y_i)$. 

\textbf{Hyper-parameters.} We set the maximum perturbation to be $\epsilon=16$ with pixel value $\in [0,255]$. For all iterative methods, we set the number of iteration steps to be $10$. For methods with momentum, we use the decay factor $\mu=1.0$. For methods related to DIM~\cite{xie2019improving}, we set the transformation probability as $0.7$. For EOT~\cite{athalye2018synthesizing}, the number of samples and optimization steps are both set to be 50.
In order to avoid influencing the performance of the attacks on untransformed images, we only consider affine transformations within a narrow range, with the settings $\theta \in [-30^{\comp}, 30^{\comp}]$, $s \in [0.5, 1.5]$, $m,n \in [-20, 20]$. Also, we set the kernel size of $\boldsymbol{\Lambda}_t$, $\boldsymbol{\Lambda}'_q$ to $(15\times15)$ and $(15\times15)$.


\begin{figure*}[htbp]
  \centering
  \subfigure[]{\includegraphics[width=0.45\textwidth]{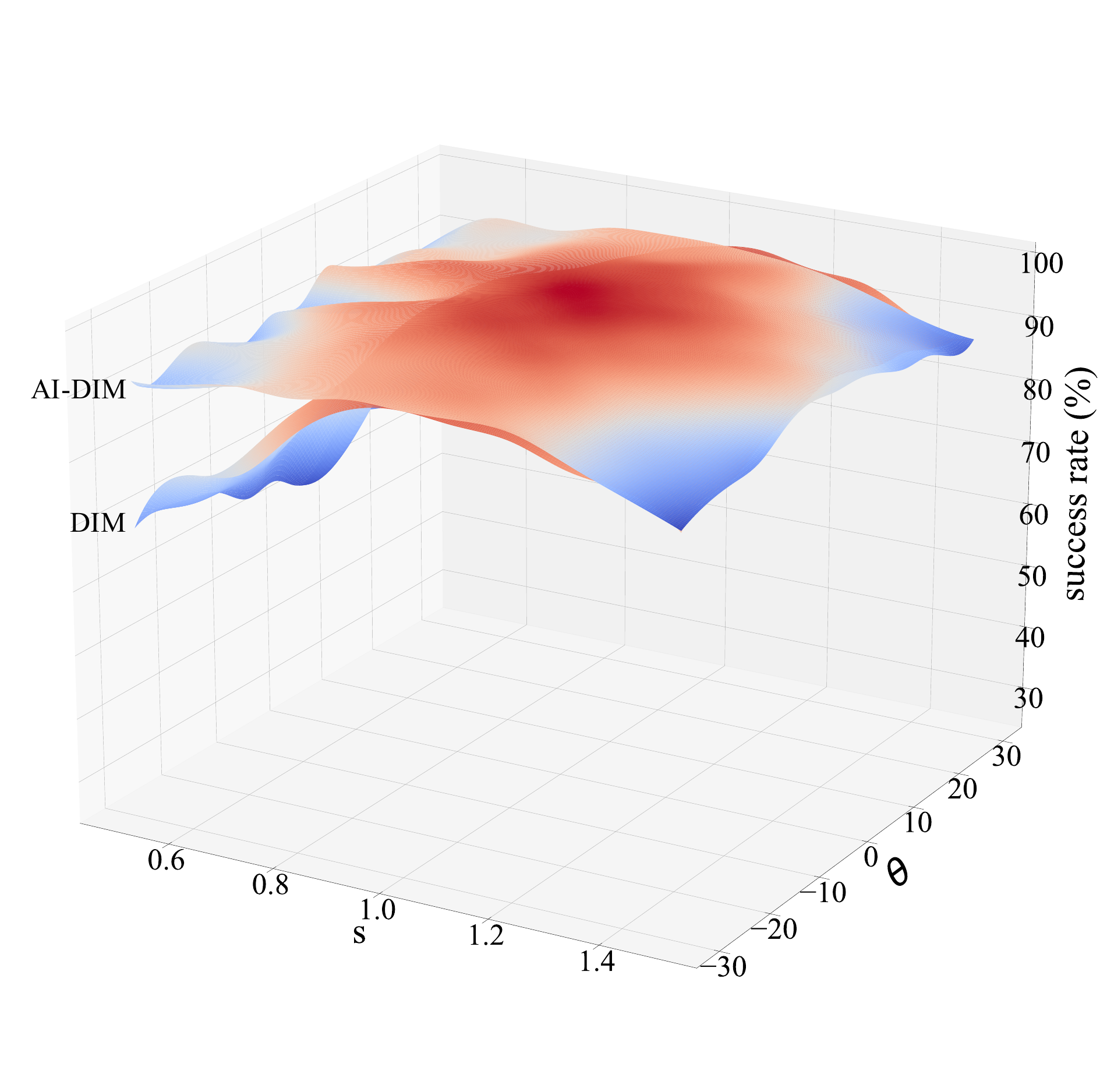} }
  \subfigure[]{\includegraphics[width=0.5\textwidth]{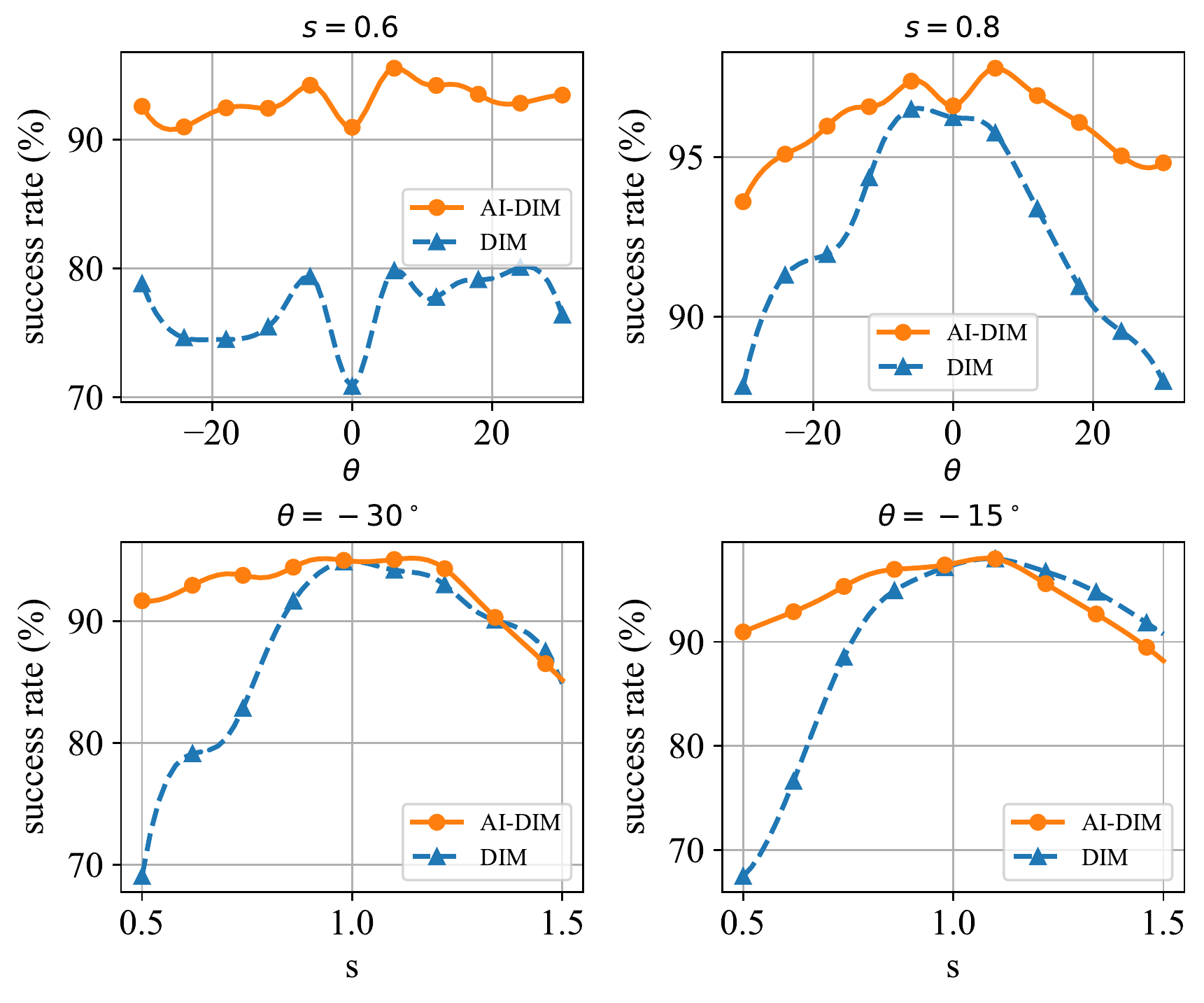} }
  \caption{The ASRs (\%) of adversarial attacks against Inc-v3 under different rotations and scalings. The adversarial examples are generated for Inc-v3 using DIM and AI-DIM. Fig. (a) shows the results in the form of a 3-D figure. Fig. (b) shows four randomly selected profiles of Fig. (a), which are $s=0.6$, $s=0.8$, $\theta=-30^{\comp}$ and $\theta=-15^{\comp}$.} 
  \label{fig:4} 
\end{figure*}

\subsection{Robustness to Affine Transformation}
\label{sec:5.2}
In this section, we show the experimental results of the proposed affine-invariant method over different affine transformations.
We selected FGSM, PGD, MIM and DIM as the basic attacks, and their extensions combined with our method are named with an "AI-" prefix. EOT is also considered as a baseline of transform-based attacks. We choose Inc-v3 as the white-box model, and test the performance on Inc-v3 (white-box model), Inc-v4 (black-box model) and Ens-AT (defense model), respectively. We separately study the ASRs in three kinds of transformations: rotation, scaling and translation.

\renewcommand{\arraystretch}{0.95}
\begin{table}[tp]
  
  \centering  
  \caption{The ASRs (\%) of adversarial attacks under different translations. The adversarial examples are crafted for Inc-v3 using the following nine attacks. We set $\theta=25^{\comp}$ and $s=0.7$, and test the performance on Inc-v3, Inc-v4 and Ens-AT.}
  \fontsize{8.}{10}\selectfont  
  \setlength{\tabcolsep}{1mm}{
    \begin{tabular}{c|c|c|c|c|c|c}  
    \hline  
    \multirow{2}{*}{Method}&  
    \multicolumn{3}{c|}{$m=5, n=5$}&\multicolumn{3}{c}{ $m=20, n=20$}\cr\cline{2-7}
    &Inc-v3&Inc-v4&Ens-AT&Inc-v3&Inc-v4&Ens-AT\cr
    \hline
    \hline
    FGSM&76.1&71.1&45.1 &76.7&70.9& 45.0\cr
  AI-FGSM&\bf87.0&\bf87.8&\bf66.0 &\bf 88.1& \bf87.4& \bf69.6\cr\hline
    PGD&43.8&32.8 &17.1 &46.6 &33.7 &21.4 \cr 
    AI-PGD&\bf85.5&\bf84.7 &\bf61.0 &\bf83.2 &\bf84.6 &\bf65.8 \cr\hline
    MIM&76.6 &72.0 &38.7 &79.3 &71.8 &41.9 \cr 
    AI-MIM&\bf91.6 &\bf90.7 &\bf73.3 &\bf91.7 &\bf90.5 & \bf70.9\cr\hline
    DIM&{85.2}&{78.1}&{43.2 }&{85.5 }&{76.8 }&{44.4}\cr
    AI-DIM&{\bf 94.1}&{\bf 92.4}&{\bf 73.7}&{\bf 93.5}&{\bf 92.2 }&{\bf 74.4}\cr
    \hline
    EOT & 74.3 & 50.9 & 14.6 & 73.4 & 51.2 & 16.9\cr
    \hline  
    \end{tabular}  
    \vspace{-1ex}
    }
    \label{tab:translation}  
\end{table} 
\renewcommand{\arraystretch}{0.95}
\begin{table}[tp]
  
  \centering  
  \caption{The average ASRs (\%) of AI-DIM and EOT and the cost of time(s) to generate 1,000 adversarial examples. The adversarial examples are crafted for Inc-v3 using the following two attacks.}
  \fontsize{8.}{10}\selectfont  
  \setlength{\tabcolsep}{2mm}{
    \begin{tabular}{c|c|c|c|c}  
    \hline  
    \multirow{2}{*}{Method}&  
    \multicolumn{3}{c|}{Avg ASRs(\%)} & \multirow{2}{*}{Avg Time(s)} \cr\cline{2-4}
    &Inc-v3&Inc-v4&Ens-AT \cr
    \hline
    \hline
    EOT&{77.9}&{51.4}& 14.8& 56580 \cr
    AI-DIM&{\bf 94.0}&{\bf 87.1}& \bf 69.5& \bf 566\cr
    \hline
    \end{tabular}  
    }
    \label{tab:eot}  
    \vspace{-2ex}
\end{table} 

\subsubsection{ Rotation}
As to rotation, we set $s=1,m=0,n=0$ and the rotation angle to be $\theta \in [-30^{\comp}, 30^{\comp}]$ at a step of $15^{\comp}$ to see the performance of different methods under different angles.
We report the test results in Tab.~\ref{tab:rotation}. In total, the ASRs increase significantly with the proposed method added to the basic models. 
Results on black-box models and defense models also demonstrate that our method is more transferable and resistant to defenses. In particular, our method brings the greatest improvement to PGD, which increases the ASR by $27\%$ on average. Furthermore, our best attack AI-DIM outperforms the EOT by a large margin, especially for black-box and defense models. For example, it improves the ASR by $57\%$ for the defense model compared to EOT. The results confirm the effectiveness of the specifically designed rotation-invariant kernel.


\begin{figure*}[htbp]
  \centering
\includegraphics[width=0.9\textwidth]{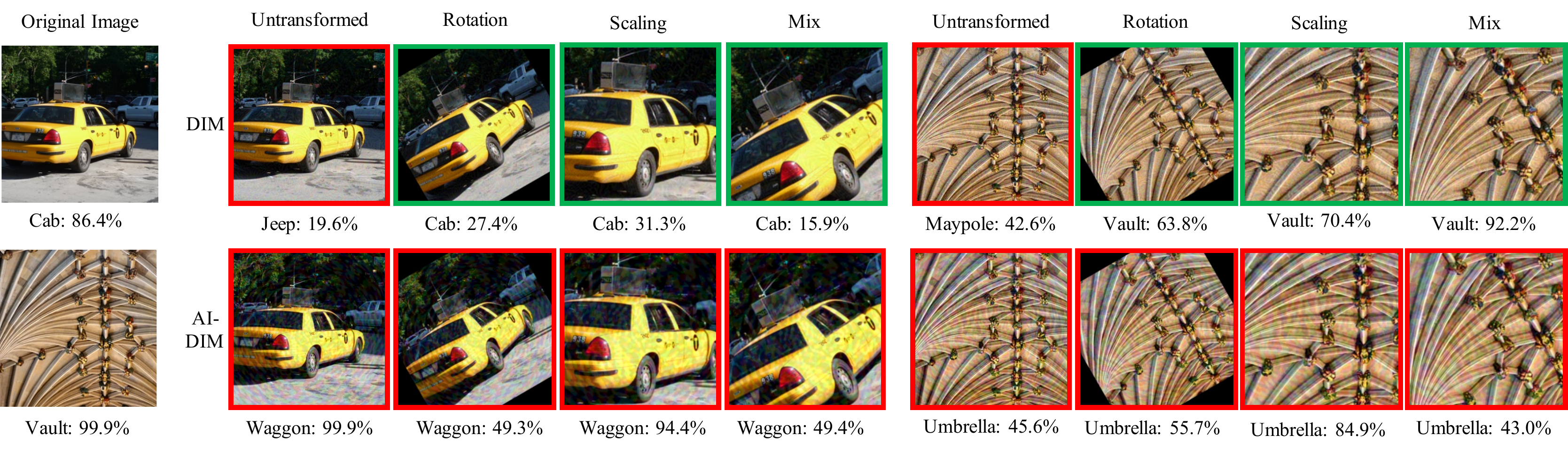}
  \caption{The adversarial examples are generated for Inc-v3 by DIM and AI-DIM with different transformations, including identity transform, rotation with $\theta=30^{\comp}$, scaling with $s = 1.5$, and mix transformation. The mix transformation consists of $\theta=30^{\comp}$, $s=1.5$ and 
  translation offsets $t=(20, 20)$. In the pictures, red represents a successful attack, and green represents a failed attack.}
  \label{fig:6} 
\end{figure*}

\subsubsection{ Scaling}
As to scaling, we perform a stress test to verify the stability of our method. We set the rotation angle and translation offsets in an extreme condition as $\theta=30^{\comp}$, $m=20$, $n=20$, and set the scaling factor as $s \in [0.5, 0.7, 1.0, 1.3, 1.5]$ to show the performance under different scales.
From the results in Tab.~\ref{tab:scaling}, we observe that our method still 
maintains better affine invariance under an extreme affine transformation, showing huge performance gains compared with the basic attacks and EOT.
Taking $s = 0.5$ in white-box attacks as an example, our method improves the ASRs by $15.3\%$, $43.3\%$, $20.7\%$ and $19.3\%$ for FGSM, PGD, MIM and DIM, respectively, and the best attack AI-DIM brings a $47.4\%$ performance gain compared to EOT. This demonstrates that our method improves the robustness to large-scale affine transformations.
The results also prove that the kernel we designed in the polar space is effective for achieving scaling invariance.

\subsubsection{ Translation}
As to translation, we randomly set the rotation angle and scaling factor as $\theta=25^{\comp}$, $s=0.7$. Due to the symmetry of translation, the translation offsets are set to positive numbers as $(m,n) \in [(5,5),(20,20)]$.
From the results in Tab.~\ref{tab:translation}, we find that different translation offsets have little effect on the attack performance. Nevertheless, 
our method still performs better than the basic attacks and EOT, demonstrating that the translation kernel we construct referring to~\cite{dong2019evading} also yields a good estimation of the gradient.

In summary, our best attack, AI-DIM, achieves an average ASR of $94.0\%$ against the white-box model, $87.1\%$ against the black-box model and $69.5\%$ against the defense model over the tested affine transformation domain. 
In order to further show the margin improved by our method, taking AI-DIM as the examples, we visualize the white-box attack-success-rate function with rotation angle and scaling factor as independent variables in Fig.~\ref{fig:4}. More results can be found in Appendix~B. We set the translation offsets both as 0, since they have little effect on the performance. 
Fig.~\ref{fig:4} show that our method keeps a relatively high attack success rate even under extreme  affine transformations, showing better affine invariance than the basic attacks. 
In addition, we compare the efficiency of our method with EOT in Tab.~\ref{tab:eot}. The experiment is conducted on a GTX 1080TI GPU. From the results, we conclude that our best method improves the attack success rate by $35.5\%$ and saves about $99\%$ on computation cost, compared to EOT.

We also visualize adversarial images generated for the Inc-v3 model by DIM, and AI-DIM with different transformations in Fig.~\ref{fig:6}, respectively. More adversarial images generated by FGSM, PGD, MIM and their corresponding combinations with our method are shown in Appendix~C. Due to transformation to polar space, we can see that the adversarial perturbations generated by our affine-invariant attacks exhibit circular patterns. Furthermore, the adversarial perturbations generated by our affine-invariant attacks are smoother than those generated by DIM, due to the smooth effect of kernel convolution.
We further show the predicted labels and probabilities for the images with different affine transformations, and the results show that the adversarial examples generated by our method are more robust to affine transformations.

\subsection{Robustness under Physical Condition}
\label{sec:5.2.4}

\begin{figure}[tp]
  \centering
  \includegraphics[width=\columnwidth]{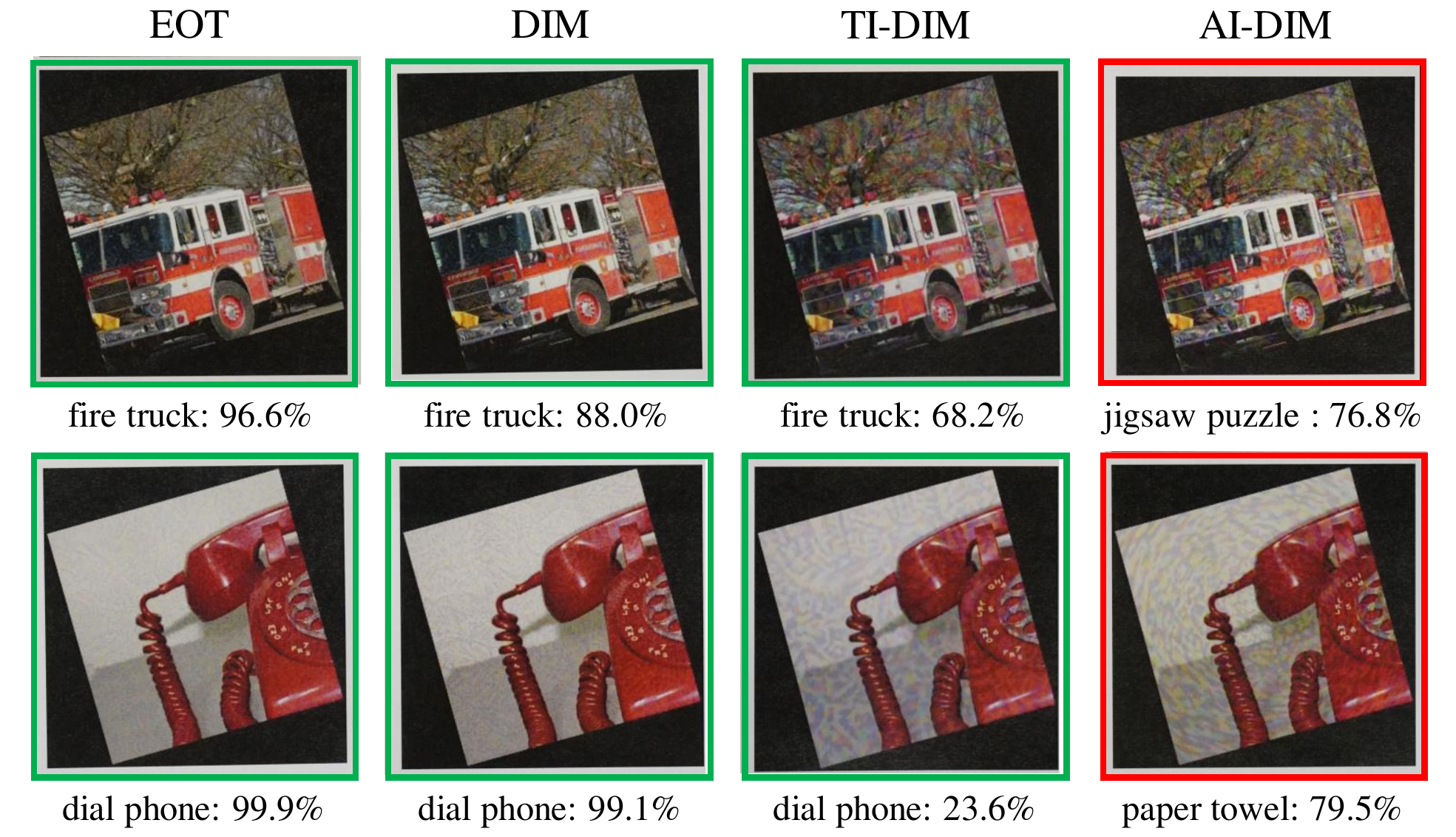} 
  \caption{The re-captured adversarial examples by camera for setting 3. These images are generated by EOT, DIM, TI-DIM and AI-DIM, respectively.} 
  \label{fig:physical_vis} 
\end{figure}

\renewcommand{\arraystretch}{1.0}
\begin{table}[tp]
  
  \centering  
  \caption{The recognition Acc. (\%) of Inc-v3 under different adversarial attacks. The adversarial examples are crafted for Inc-v3, and recaptured by camera after printing. }
  \vspace{1ex}
  \fontsize{8.5}{10}\selectfont  
  \setlength{\tabcolsep}{1mm}{
    \begin{tabular}{c|c|c|c}  
    \hline  
    Method & setting 1 & setting 2 & setting 3 \cr
    \hline
    \hline
    EOT&24.3 & 21.6 &18.0  \cr\hline
    DIM & 11.8 &12.7 &15.3  \cr\hline
    TI-DIM& 10.7 & 8.5 &7.6 \cr\hline
    AI-DIM&\bf 10.5 &\bf 5.5 &\bf 5.6 \cr
    \hline
    \end{tabular}  
    }
    \label{tab:physical_test}  
\end{table} 

\renewcommand{\arraystretch}{0.95} 
\begin{table*}[tp]  
  
  \centering  
  \caption{The ASRs (\%) of adversarial attacks against six defense models. The adversarial examples for single-model attacks and ensemble-based attacks are crafted for IncRes-v2 and the ensemble of Inc-v4, IncRes-v2, and Res-v2-152, respectively, using FGSM, PGD, MIM, DIM and their TI and AI extensions.
  }  
  \vspace{1ex}
  \fontsize{8.}{10}\selectfont  
  \setlength{\tabcolsep}{1.0mm}{
    \begin{tabular}{c|c|c|c|c|c|c|c|c|c|c|c|c}  
    \hline  
    \multirow{2}*{Method}&  
    \multicolumn{6}{c|}{Single-model attacks}&\multicolumn{6}{c}{ Ensemble-based attacks}\cr\cline{2-13}  
    &Ens-AT&ALP&JPEG&Bit-Red&R\&P&RandMix&Ens-AT&ALP&JPEG&Bit-Red&R\&P&RandMix\cr
    \hline
    \hline
    FGSM &16.1&6.4&37.6&30.0&48.7&35.0&27.1&11.1&49.1&41.8&58.7&40.7\cr
    TI-FGSM&{ 26.1}&21.6&46.9 &40.7&56.5&\bf 42.2 &{34.2}&22.8&\bf53.9 & 47.4&\bf60.2& \bf 44.8  \cr
    AI-FGSM&{\bf 31.7}&{\bf 30.7}&{\bf48.6} &{\bf 43.8}&{\bf 57.2}& { 42.0}&{\bf 34.9}&{\bf 31.7}&{\bf49.8} &{\bf 47.8}&{59.4}& {40.2} \cr
    \hline
    PGD&8.4&0.4&26.7 &19.4&36.6&12.7 &25.8&0.4&58.0 &32.8&66.8&23.7  \cr
    TI-PGD&18.8&9.7&29.3 &19.9&36.2&18.5 &41.4&10.1&56.0 &31.9&59.4&23.2  \cr
    AI-PGD&{\bf 31.2}&{\bf 21.8}&{\bf48.0} &{\bf 40.4}&{\bf 57.4}&{\bf35.2}&{\bf 54.6}&{\bf 27.6}&{\bf72.9} &{\bf 53.2}&{\bf 75.8}&{\bf41.5} \cr
    \hline
    MIM&34.5& 5.1 &60.3 &44.6&68.6&40.2 & 59.1&{ 10.3}&80.6 &62.8&83.4&51.8  \cr
    TI-MIM&46.7&{\bf33.7}&56.9 &45.5&61.6&38.0& 65.3&{\bf 35.2}&73.6 &56.1&76.4&42.7  \cr
    AI-MIM&{\bf 49.4}&{ 31.9}&{\bf65.8} &{\bf 52.0}&{\bf 70.1}&{\bf48.5}&{\bf 69.6}&{ 33.7}&{\bf81.4} &{\bf 63.9}&{\bf84.9}&{\bf54.3} \cr
    \hline
    DIM&{ 54.7}&{ 8.6}&{ 71.3 }&{ 58.1}&{ 77.1}&{\bf 50.0 }&{ 80.2}&{ 14.2}&{83.4 }&{ 69.8}&{ 83.8}&{52.5 } \cr  
    TI-DIM&{ 60.9}&{ 35.0}&{ 70.1 }&{ 53.3}&{ 72.9}&{ 43.2 }&{ 78.0}&{ 39.9}&{84.5 }&{ 65.8}&{ 85.8}&{51.3 } \cr  
    AI-DIM&{\bf 62.7}&{\bf 36.4}&{\bf 76.2}&{\bf 59.8}&{\bf 78.8}&{\bf 50.0 }&{\bf 82.1}&{\bf 42.6}&{\bf 89.1}&{\bf 72.2}&{\bf 90.7}&{\bf 60.7} \cr 
    \hline  
    \end{tabular}  
    }
  \label{tab:transferability}
\end{table*}

\renewcommand{\arraystretch}{0.95}
\begin{table*}[tp]
  \centering  
  \caption{The ASRs (\%) and average queries of score-based attacks with different initializations. The surrogate models for the initialization are the ensemble of Inc-v4, IncRes-v2 and  Res-v2-15. The selected defense models are Ens-AT, ALP, JPEG, Bit-Red, R$\&$P and RandMix.}
  \vspace{1ex}
  \fontsize{8.}{10}\selectfont  
  \setlength{\tabcolsep}{1.0mm}{
    \begin{tabular}{c|c|c|c|c|c|c|c|c|c|c|c|c}  
    \hline  
    \multirow{2}{*}{Method}&  
    \multicolumn{2}{c|}{Ens-AT}&\multicolumn{2}{c|}{ALP}&\multicolumn{2}{c|}{JPEG}&\multicolumn{2}{c|}{Bit-Red}&\multicolumn{2}{c|}{R$\&$P}&\multicolumn{2}{c}{RandMix}\cr\cline{2-13}
    & ASR & Avg Q & ASR & Avg Q & ASR & Avg Q& ASR & Avg Q & ASR & Avg Q & ASR & Avg Q\cr
    \hline
    \hline
    NES& 95.7 &1827& 82.7 &1425& 55.8 & 6126 & 96.8&1132 & 8.4& 4236& 1.5 & 3684 \cr
    NES-PGD& 97.3& 1045 & 82.3& 1377 & 90.9& 1688 & 98.5& 443& 65.2& 895& 47.0 & 2658 \cr
    NES-TI-DIM& 98.9& 315 &88.1 & 939 & 95.5& 670 & 99.2& 242& 86.2 & 499& 52.0& 1524 \cr 
    NES-AI-DIM&\bf 99.4&\bf 183 &\bf 88.5&\bf 937 &\bf 97.1&\bf 423 & \bf 99.7& \bf154 & \bf91.7 & \bf246 & \bf 60.7&\bf 1069 \cr
    \hline
    SPSA& 96.9 &1516& 80.5 &1556& 52.0 & 6001 & 96.8& 994& 8.8& 3735& 0.7& 3682 \cr
    SPSA-PGD& 97.8& 910 & 81.9& 1441 & 87.5& 1753 & 98.6& 414& 68.6& 863& 25.2& 2690\cr
    SPSA-TI-DIM& 99.3& 265 &88.1 & 1013 & 95.5&  637 & 99.2& 222 & 86.0& 476& 51.2& 1557\cr 
    SPSA-AI-DIM&\bf 99.6&\bf 167 &\bf 88.7&\bf 1006 &\bf 96.6&\bf 410 & \bf99.7& \bf140& \bf90.5& \bf223& \bf61.3& \bf1083\cr
    \hline
    $\mathcal{N}$ATTACK& 99.1 & 805 & 98.6 & 505& 97.5 & 1057 & 99.6 & 577& 31.9& 1529& 7.8& 2487\cr
    $\mathcal{N}$ATTACK-PGD& 99.7& 395 & 98.8& 450 & 97.8& 244& 99.7& 274& 73.6& 340& 27.0& 1699\cr
    $\mathcal{N}$ATTACK-TI-DIM& 99.8& 147 &\bf99.0 &\bf 250 & 99.3& 126 & 99.5& 152& 89.2& 271& 53.8& 1176\cr 
    $\mathcal{N}$ATTACK-AI-DIM&\bf 99.9&\bf 87 & 98.8& 277 &\bf 99.7&\bf 58 &\bf 99.9&\bf 97&\bf 93.0&\bf 131&\bf 63.0&\bf 838\cr
    \hline
    \end{tabular}  
    }
    \label{tab:score}  
\end{table*} 

To further exhibit the performance of our method under physical experiment condition, we print all the 1,000 adversarial images and obtain the affine-transformed test data by adjusting the camera position parameters \etc. 
With the disturbance of physical conditions such as lighting, the classifier's recognition accuracy for recaptured images will decrease a lot. Therefore, in this part, we will narrow down the transformation ranges such as limiting rotations within 15 degrees. We show the classifier's recognition accuracy with different attack methods and three transformation settings in Tab.~\ref{tab:physical_test}, which are denoted as setting 1: $\theta=0^{\comp}, s=1.0, m=0, n=0$, setting 2: $\theta=5^{\comp}, s=0.9, m=0, n=0$ and setting 3: $\theta=15^{\comp}, s=0.8, m=0, n=0$. Here we use the recognition accuracy as the performance indicator since $M$ in ASR will be further affected under physical condition, and recognition accuracy is more objective. Due to the inevitable deviations in the shooting process, there will be some random offsets in each transformations, such as translation offset from 0 to 20 pixels. In Fig.~\ref{fig:physical_vis}, we visualize some re-captured adversarial examples by camera for setting 3. The attack hyper-parameters are the same as the previous ones, and the source model and test model are both Inc-v3. Tab.~\ref{tab:physical_test} shows that even under physical condition, our method outperforms the rest three attacks, which further verifies the robustness of the proposed affine-invariant attacks to affine transformations.

\subsection{Robustness to Defense Models under Black-box Setting}
\label{sec:5.3}

TI~\cite{dong2019evading} has shown that it can improve the transferability of adversarial examples greatly with respect to the defense models. As shown in Sec.~\ref{sec:3}, the proposed affine-invariant method is an enhancement of the TI method. Therefore, we conduct an experiment to show the transferability of adversarial examples generated by different attacks against defense models. We test the performance of single-model attacks and ensemble-based attacks~\cite{dong2018boosting}, respectively. For single-model attacks, we set IncRes-v2 as the surrogate model to generate adversarial examples. As 
for ensemble-based attacks, we attack the ensemble of Inc-v4, IncRes-v2, and Res-v2-152 with equal ensemble weights. Furthermore, we choose six state-of-the-art defense models according to RealSafe~\cite{dong2020benchmarking}.

From Tab.~\ref{tab:transferability}, we can see that, compared with the TI method and basic attacks, our method yields a significant improvement for tested defense models. In particular, combined with PGD, MIM and DIM, our method improves the ASRs by 
$17.1\%$, $6.2\%$, and $5.1\%$, respectively,
on average for ensemble-based attacks, compared to the TI method. It demonstrates that the proposed affine-invariant attacks can better improve the tranferability of the generated adversarial examples to evade the defense models. The primary reason is that our method considers a wider transformation domain, and can generate adversarial examples that are less sensitive to the discriminative regions of the white-box model, helping to evade the defense models~\cite{dong2019evading}.

To further verify the transferability, we set our method as the initialization of some scored-based black-box attacks, and compare its performance with PGD, TI and the original attacks. We choose NES~\cite{ilyas2018black}, SPSA~\cite{uesato2018adversarial} and $\mathcal{N}$ATTACK~\cite{li2019nattack} as the score-based black-box attacks; Ens-AT, ALP, JPEG, Bit-Red, R$\&$P and RandMix as the defense models; and the ensemble of Inc-v4, IncRes-v2, Res-v2-15 as the surrogate models for the initialization. The maximum number of queries and magnitude are set to be 10,000 and 16. The results in Tab.~\ref{tab:score} demonstrate that our method not only increases the ASR, but also greatly reduces the required number of queries by up to $95\%$, an outcome that undoubtedly is meaningful to black-box attacks.

\renewcommand{\arraystretch}{0.9}
\begin{table*}[tp]
  \centering
  \vspace{2ex}
  \caption{The ASRs (\%) of AI-DIM with two different kernel types under different affine transformations against one white-box model (Inc-V3), one black-box model (Inc-V4) and six defense models (Ens-AT, ALP, JPEG, Bit-Red, R$\&$P and RandMix). The adversarial examples are crafted for Inc-V3.}
  \fontsize{7.5}{10}\selectfont  
  \setlength{\tabcolsep}{1.0mm}{
    \begin{tabular}{c|c|c|c|c|c|c|c|c|c|c|c|c|c|c|c|c}  
    \hline  
    \multirow{2}{*}{Transformation Sample}&  
    \multicolumn{2}{c|}{Inc-v3}&\multicolumn{2}{c|}{Inc-v4}&\multicolumn{2}{c|}{Ens-AT}&\multicolumn{2}{c|}{ALP}&\multicolumn{2}{c|}{JPEG}&\multicolumn{2}{c|}{Bit-Red}&\multicolumn{2}{c|}{R$\&$P}&\multicolumn{2}{c}{RandMix}\cr\cline{2-17}
    & Uni. & Gau. & Uni. & Gau.& Uni. & Gau. & Uni. & Gau. & Uni. & Gau.& Uni. & Gau. & Uni. & Gau. & Uni. & Gau.\cr
    \hline
    \hline
    $(\theta,s,m,n)=(0^{\comp},1.0,0,0)$& 91.9 & \bf99.8 &48.5 &\bf 69.3&\bf 52.3 & 39.2 &\bf 45.0 &38.1 & 88.8&\bf 99.1& 65.8 & \bf85.0 & 88.9 &\bf 98.6 &37.4 &\bf 56.8 \cr

    $(\theta,s,m,n)=(15^{\comp},1.2,5,5)$& 80.9 &\bf 95.8& 61.1&\bf 76.3 &\bf 67.0& 64.8 &\bf 61.6& 44.5& 79.6&\bf 93.5& 67.5 &\bf 77.3& 85.8 &\bf 96.1 &46.2 &\bf 62.6 \cr

    $(\theta,s,m,n)=(30^{\comp},1.5,20,20)$& 74.6 &\bf 87.5& 65.3 &\bf 79.0 &\bf 72.0& 71.3 &\bf 68.7& 56.4& 75.0 &\bf 85.8& 64.2 &\bf 75.1& 78.2&\bf 88.1 &49.2 &\bf 59.0 \cr

    $(\theta,s,m,n)=($-$10^{\comp},0.7,5,5)$& 90.7 &\bf 95.4 & 89.0 &\bf 91.7 &\bf 85.2& 73.6 & \bf 65.1& 57.1 & 86.9 & \bf 90.8 & 75.6&\bf 79.0 & 89.9&\bf 93.6 & 57.4&\bf 72.3 \cr
    
    $(\theta,s,m,n)=($-$20^{\comp},0.5,20,20)$&\bf 87.1 &\bf 87.1 & 89.6 &\bf 91.6 &\bf 81.9& 66.7 & \bf 46.7& \bf 46.7 & \bf 83.2 & 82.6 & 71.3 &\bf 72.1 &88.0 &\bf 88.8 &25.0 &\bf 42.9 \cr
    \hline
    \end{tabular}
    }
    \label{tab:kernel_type}  
\end{table*} 

\begin{figure*}[htbp]
  \centering
  \includegraphics[width=.265\textwidth]{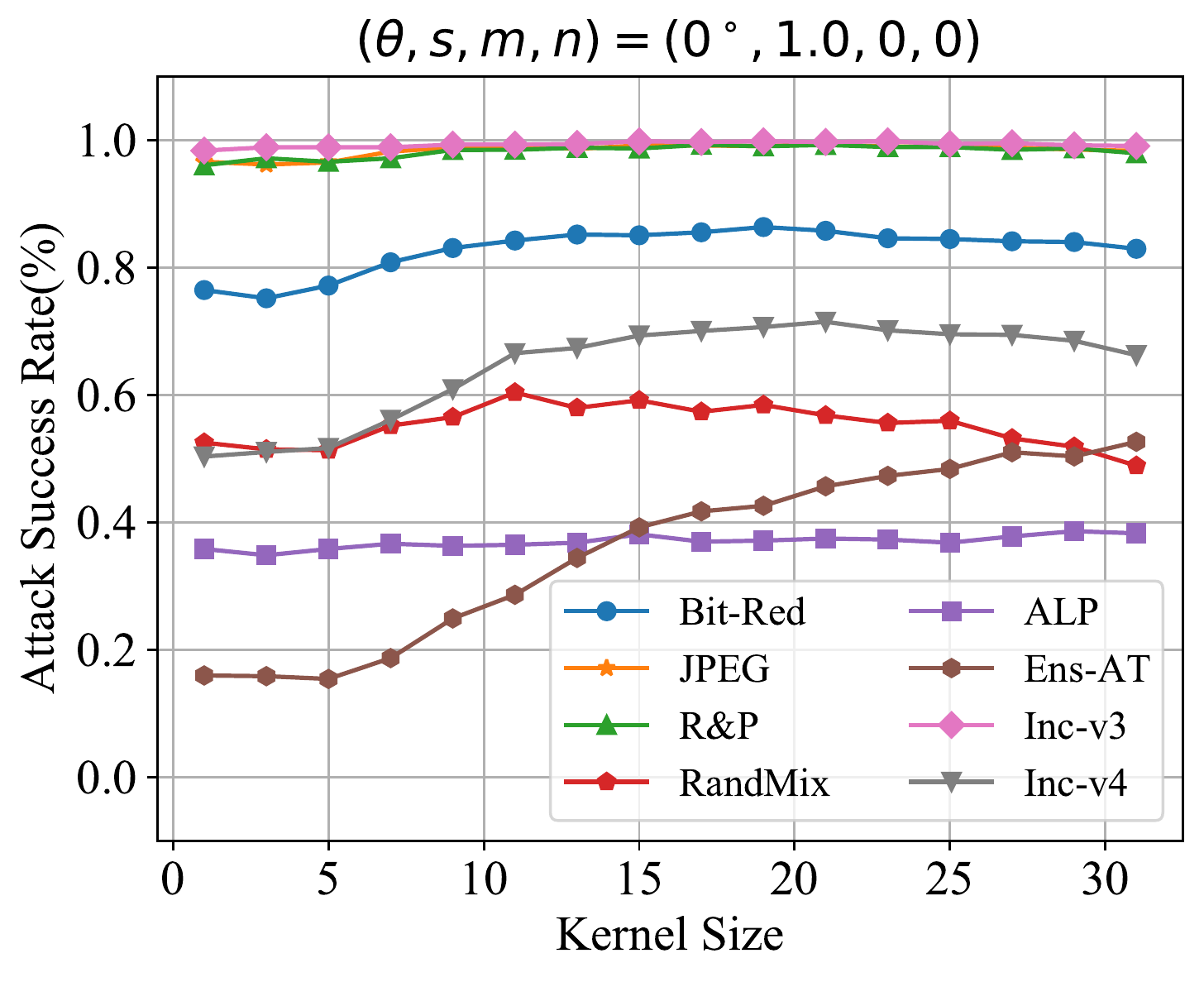}
  \includegraphics[width=.265\textwidth]{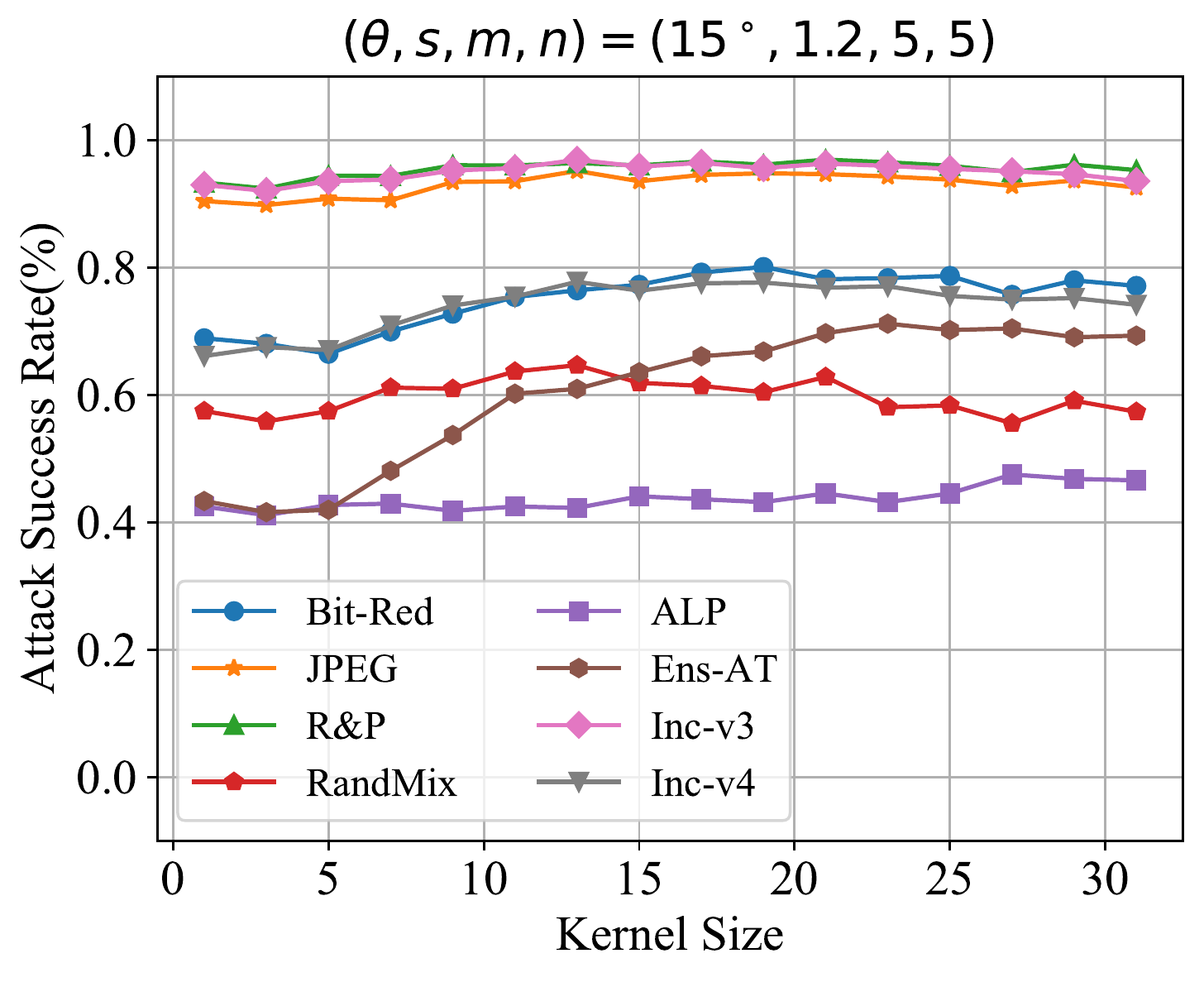}
  \includegraphics[width=.265\textwidth]{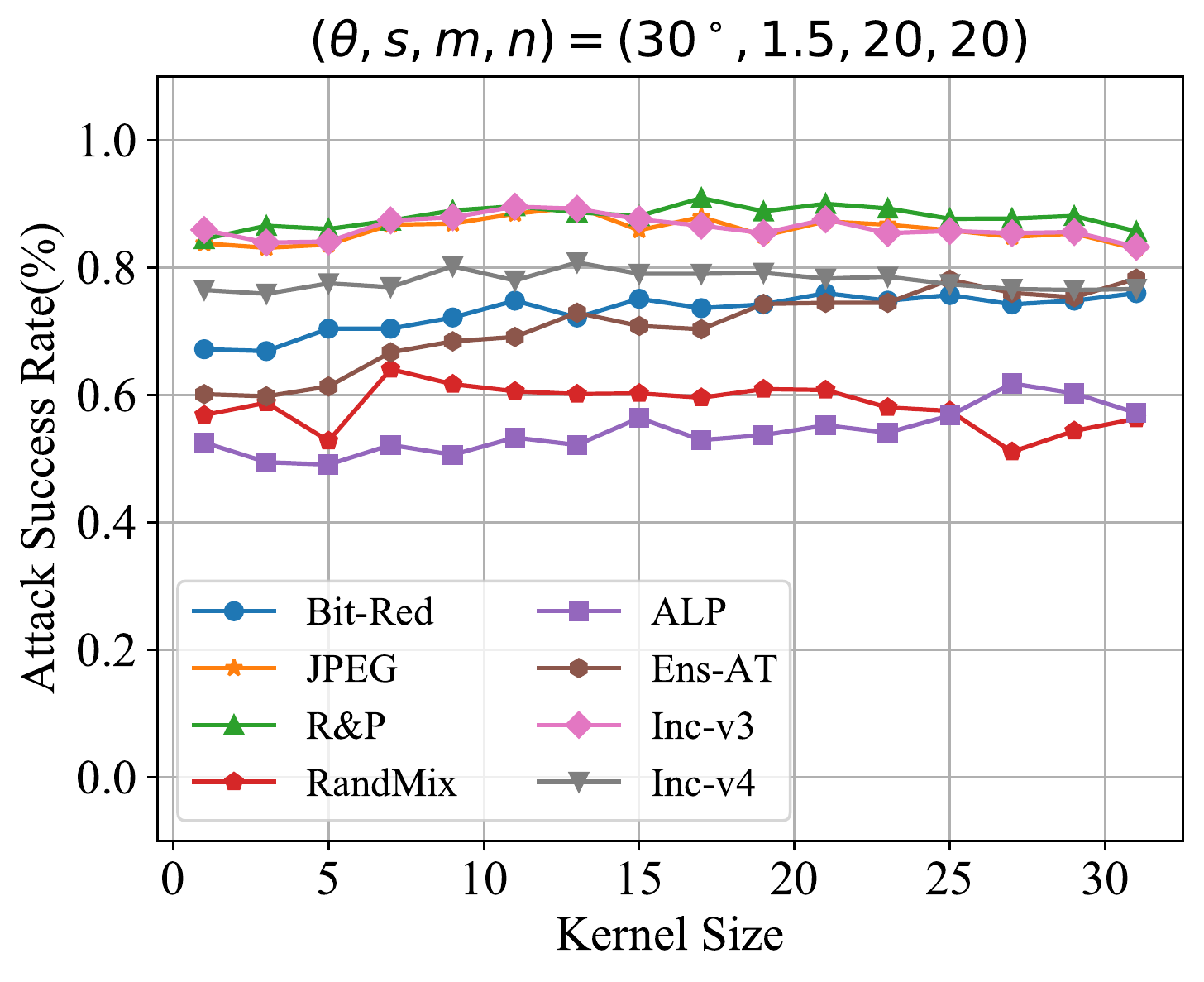}
  \includegraphics[width=.265\textwidth]{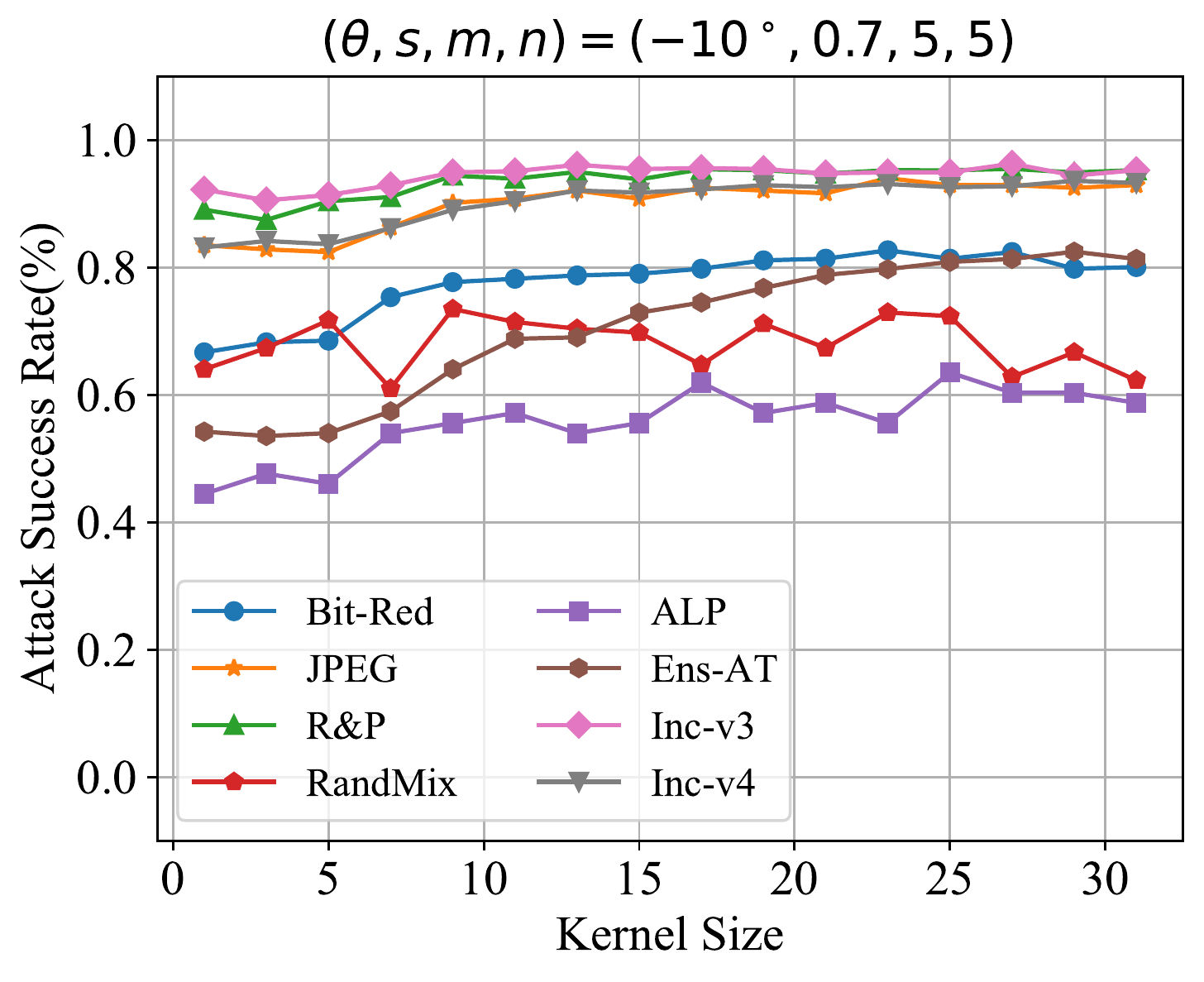}
  \includegraphics[width=.265\textwidth]{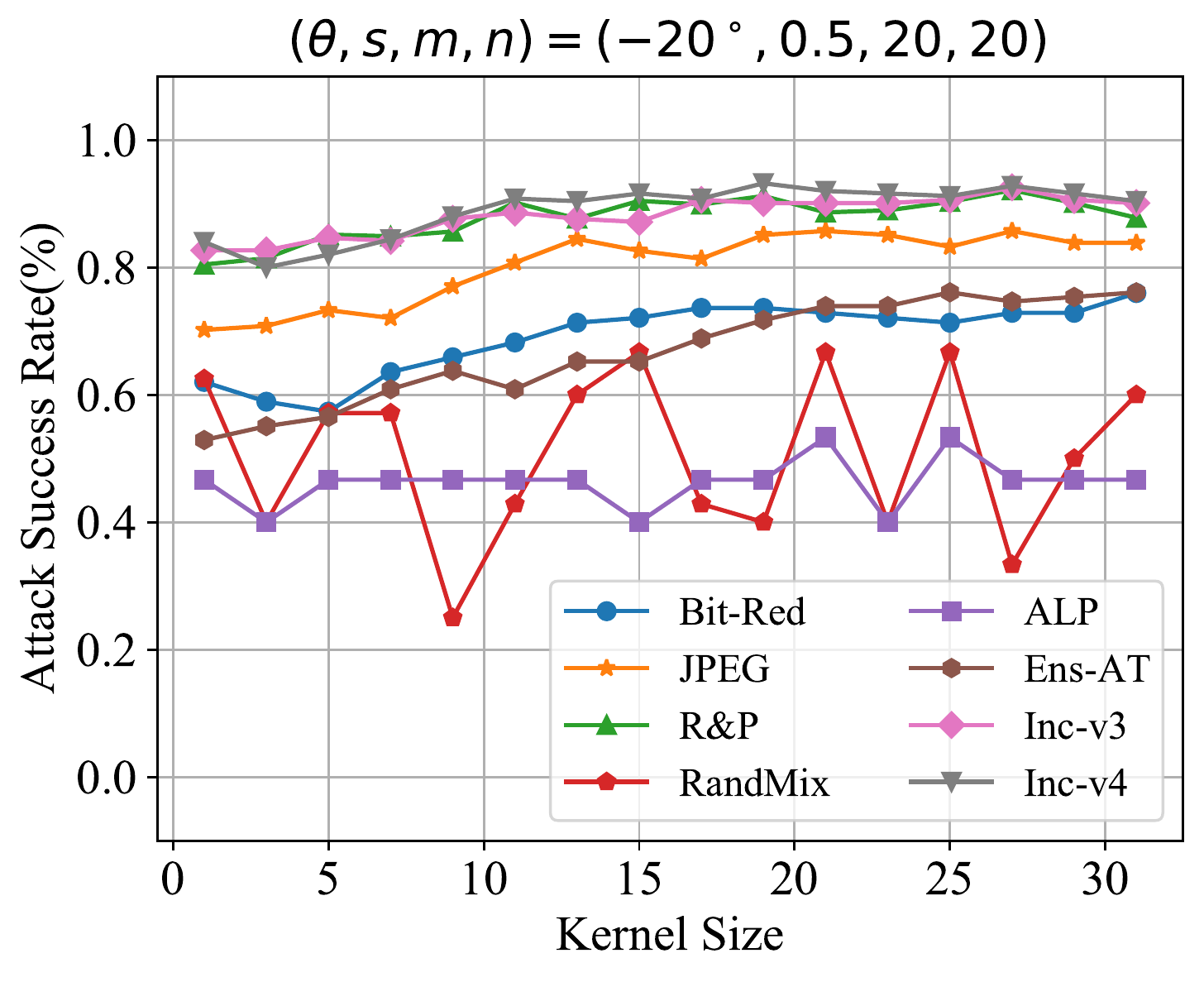}
  \caption{The ASRs (\%) of AI-DIM with different kernel sizes ranging from 1 to 31 under different affine transformations against one white-box model (Inc-V3), one black-box model (Inc-V4) and six defense models (Ens-AT, ALP, JPEG, Bit-Red, R$\&$P and RandMix). The source model are also Inc-V3.}
  \label{fig:7} 
\end{figure*}

\subsection{Ablation Study}
\label{sec:5.4}
In this section, we delve into the proposed affine-invariant gradient estimator to explore the effect of each component. Overall, our proposed affine-invariant attack can be regarded as an enhancement of the basic attack and TI attack. From the experimental results in Sec.~\ref{sec:5.2} and Sec.~\ref{sec:5.3}, the affine-invariant gradient estimator can undoubtedly improve the affine invariance and black-box ASRs on defense models of the generated adversarial examples. Delving into detailed components, we focus on the effect of convolutional kernel type and size.

\subsubsection{Kernel Type}
In order to verify the principle we have adopted that the more the  input image changes, the lower weight it should be assigned, we design another Uniform kernel to compare with the Gaussian kernel, meaning that we set the same weight for each transformed image. Tab.~\ref{tab:kernel_type} shows the performance of AI-DIM under five random selected affine transformations to one white-box model, one black-box model and six defense models. From the results, we can see that, except for very few defense models, the Gaussian kernel performs better on most models. The results also confirm our analysis in Sec.~\ref{sec:4}. 

\subsubsection{Kernel Size}
We set the kernel type as Gaussian, and further investigate the effect of kernel size. To ensure the simplicity of the experiment, we keep the size of the two kernels equal. Fig.~\ref{fig:7} shows the results obtained under the same experimental settings as the experiment concerning kernel type. From the results, we find that at first the attack success rates are positively correlated to the kernel size, but gradually tend to be stable or even descend after the kernel size reaches 15. Nonetheless, there exist some models such as Ens-AT on which the attack success rates keep increasing when the kernel size increases. In general, our method performs better when the kernel size is around 15. Therefore, the kernel size is also set to be $15\times15$ in our main experiments.

\subsection{Discussion}
\label{sec:5.5}
From a traditional perspective, affine invariance and black-box transferability of adversarial examples are two completely unrelated concepts. In our method, we only make use of the gradients of untransformed images, instead of other transformed ones. Considering the affine transformation module as part of the target model, we can regard the situation without affine transformation as the white-box setting and others as the black-box setting. Therefore, this type of affine invariance can actually be regarded as the transferability on the affine transformation. We collectively refer to these two transferabilities as the generalized transferability, referring to the generalization and robustness of adversarial attacks when facing unknown environments. With our method, the proposed affine-invariant attacks improve not only the transferability over black-box models, but also that over affine transformation. Therefore, we declare that the affine-invariant attacks further enhance the generalized transferability of adversarial examples.

\section{Conclusion}
\label{sec:7}
In this paper, we propose an affine-invariant attack method to improve affine-invariance and transferability of adversarial examples. 
Our method optimizes adversarial perturbations by a gradient estimator, providing an estimation of the affine-invariant gradient, accelerated with convolution operations. Additionally, we provide an analysis of the gradient approximation error.
Our method can be integrated into any gradient-based attack methods. We conducted extensive experiments to validate the effectiveness of the proposed method. Our best attack, AI-DIM, achieves an average success rate of $94.0\%$ against the white-box model, $87.1\%$ against the black-box model, and $69.5\%$ against the defense model under the tested affine transformations. Compared with EOT, our method yields a $30\%$ higher ASR with only about $1\%$ of the computation cost. Also, we design physical experiments and statistically show our method is more robust to complex transformations in the physical world. Furthermore, our method improves the success rate by an average of $7.5\%$ over the state-of-the-art transfer-based black-box attack on six defense models. Notably, the best method reduces the number of queries by up to $95\%$ for the tested score-based black-box attack. 

{\small
\bibliographystyle{ieee_fullname}
\bibliography{egbib}
}

\clearpage
\noindent \begin{center} {\large  \textbf{Appendix}} \end{center}

\begin{appendix}
\setcounter{corollary}{0}

\begin{figure*}[htbp]
  \centering
  \subfigure[]{\includegraphics[width=0.45\textwidth]{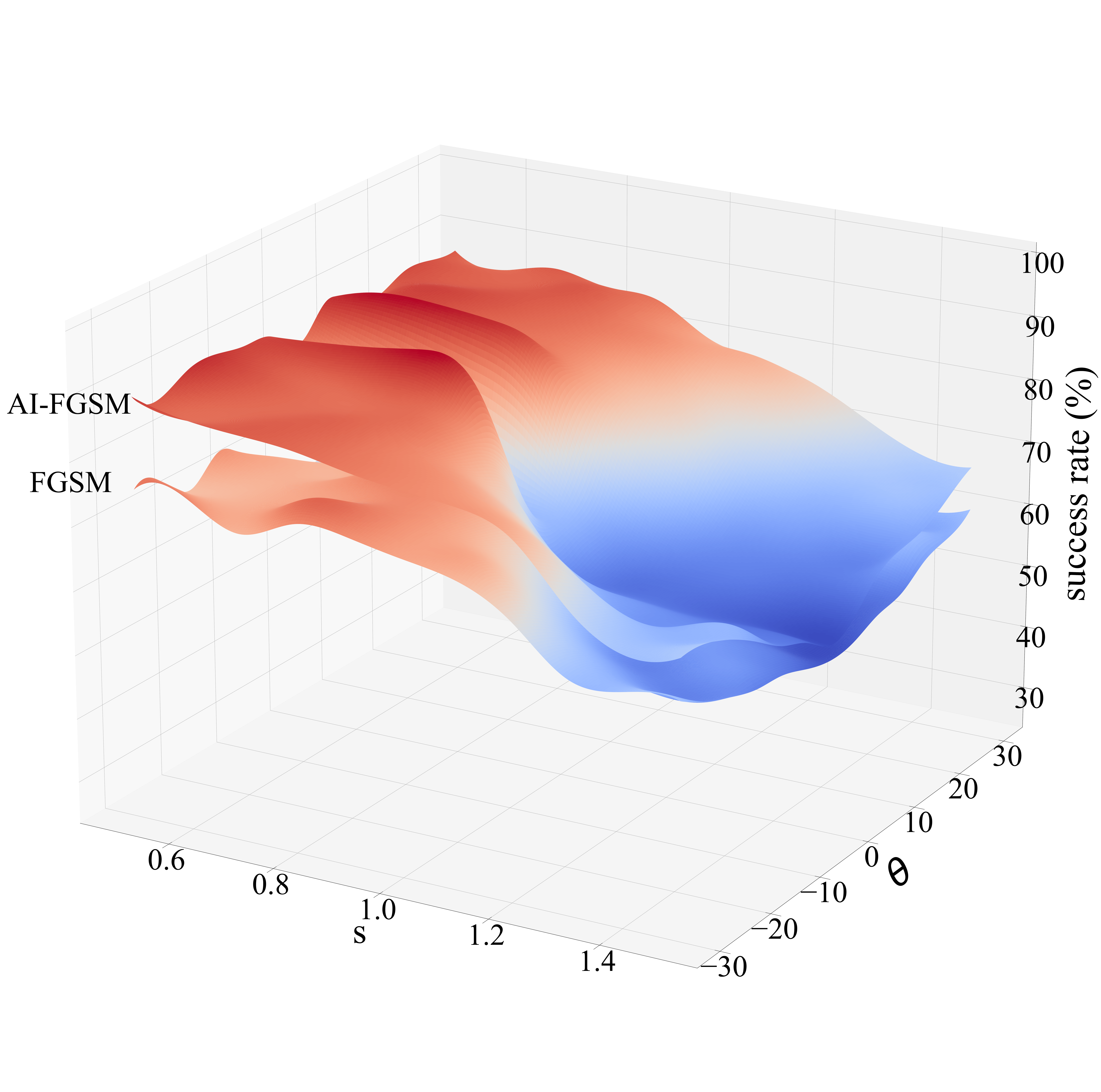}}
  \subfigure[]{\includegraphics[width=0.5\textwidth]{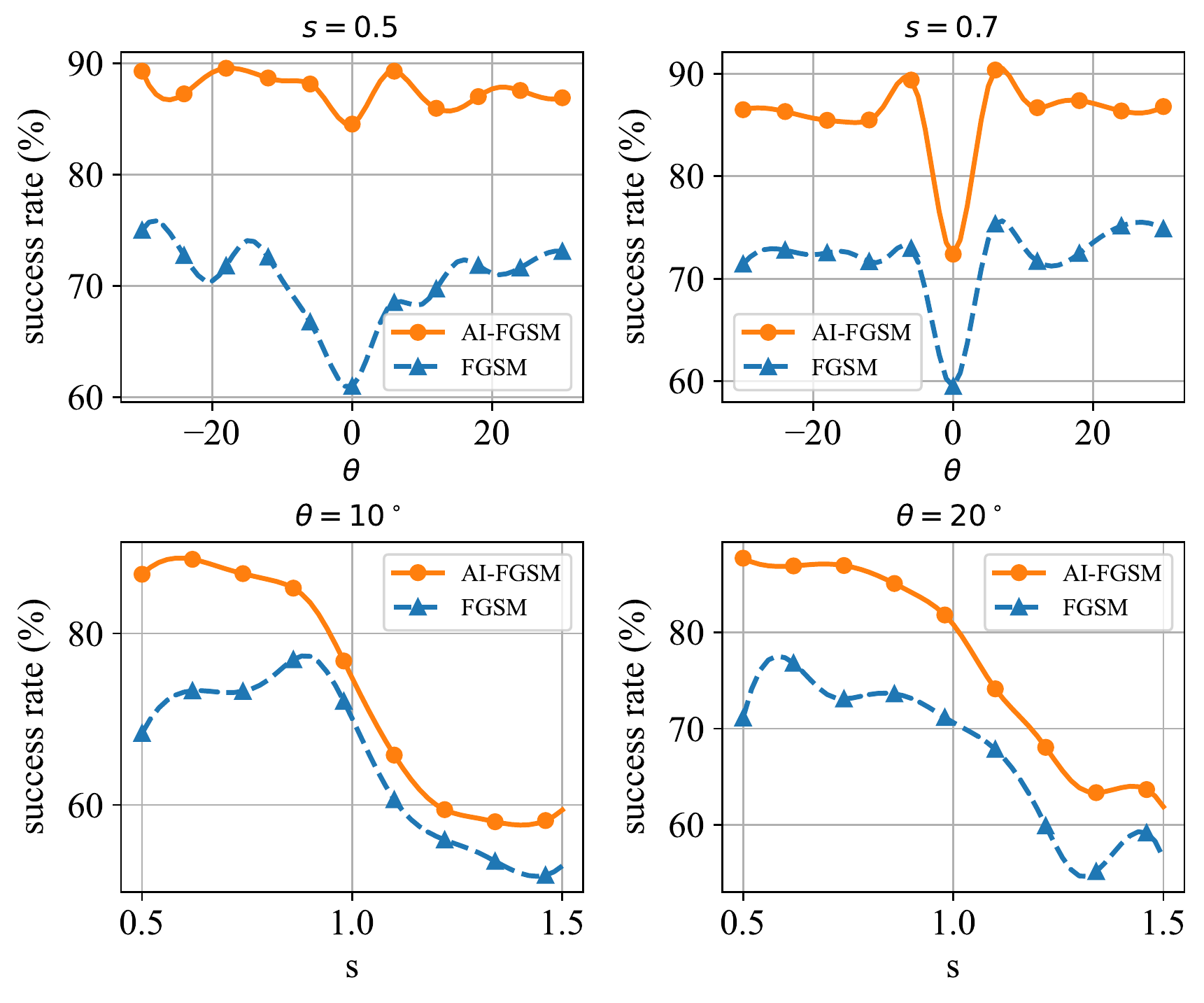}}
  \caption{The ASRs (\%) of adversarial attacks against Inc-v3 under different rotations and scalings. The adversarial examples are generated for Inc-v3 using FGSM and AI-FGSM. Fig. (a) shows the results in the form of a 3-D figure. Fig. (b) shows four randomly selected profiles of Fig. (a), which are $s=0.5$, $s=0.7$, $\theta=10^{\comp}$ and $\theta=20^{\comp}$.} 
  \label{fig:8} 
\end{figure*}
\begin{figure*}[htbp]
  \centering
  \subfigure[]{\includegraphics[width=0.45\textwidth]{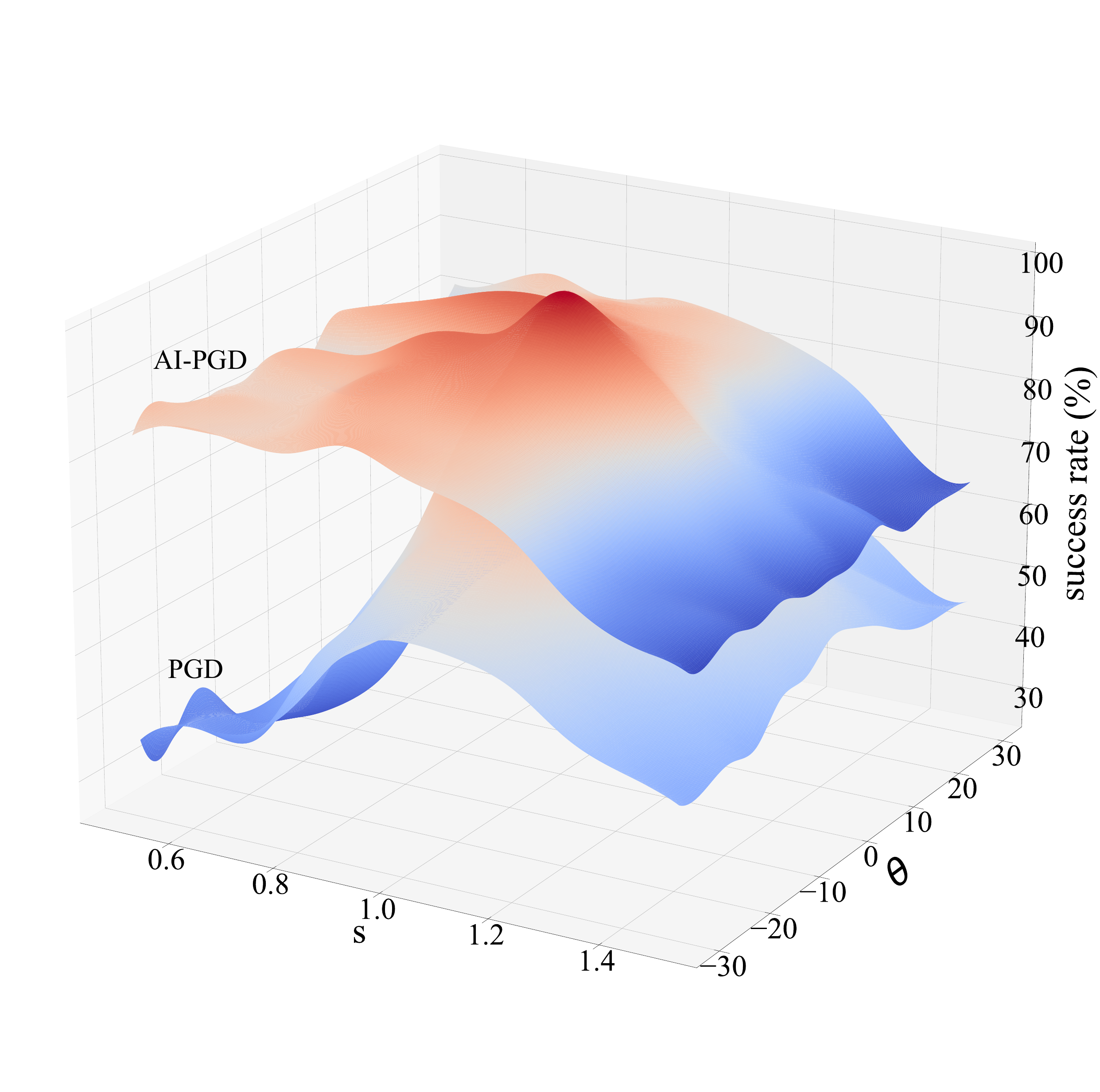} }
  \subfigure[]{\includegraphics[width=0.5\textwidth]{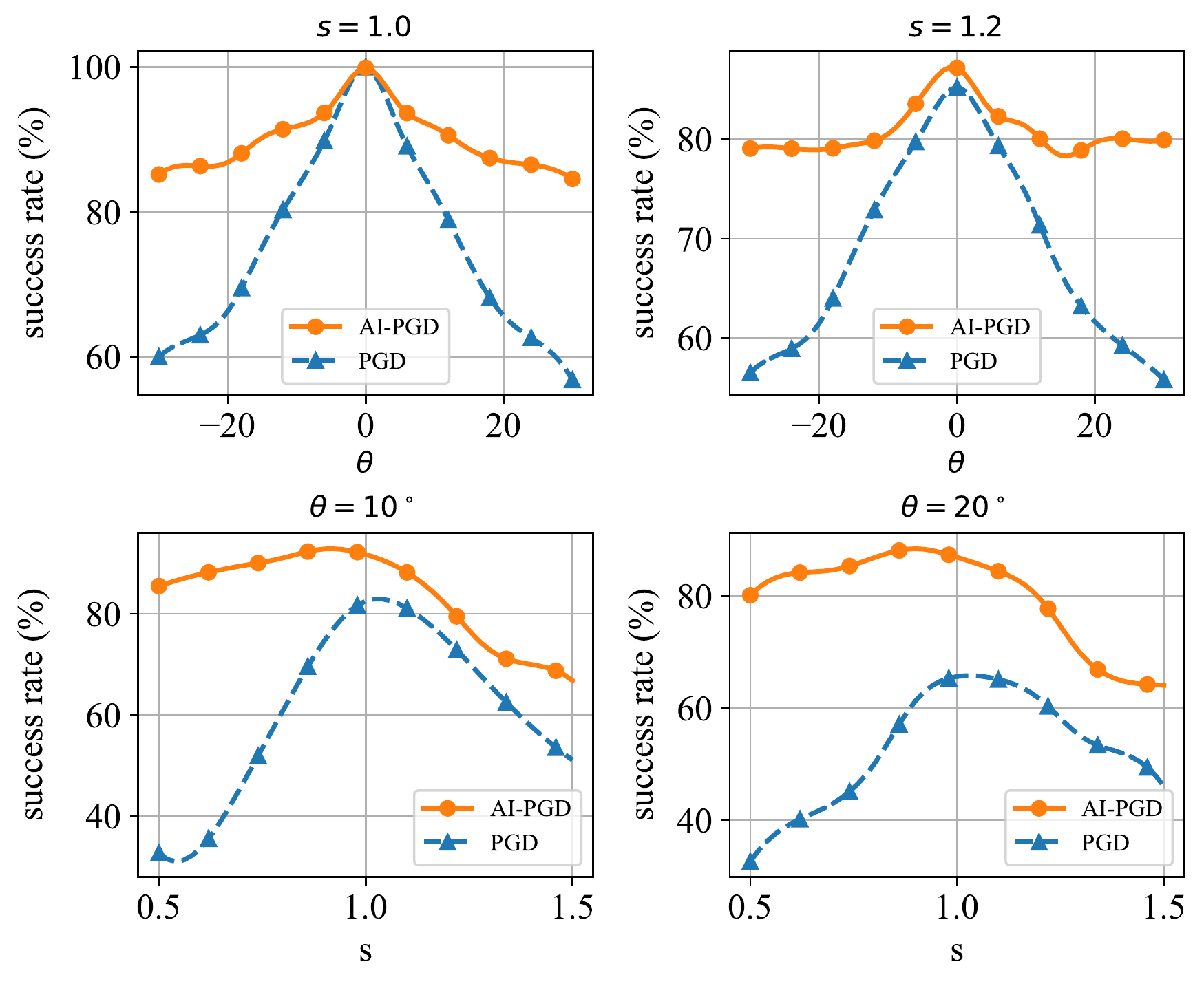} }
  \caption{The ASRs (\%) of adversarial attacks against Inc-v3 under different rotations and scalings. The adversarial examples are generated for Inc-v3 using PGD and AI-PGD. Fig. (a) shows the results in the form of a 3-D figure. Fig. (b) shows four randomly selected profiles of Fig. (a), which are $s=1.0$, $s=1.2$, $\theta=10^{\comp}$ and $\theta=20^{\comp}$.} 
  \label{fig:3} 
\end{figure*}
\begin{figure*}[htbp]
  \centering
  \subfigure[]{\includegraphics[width=0.45\textwidth]{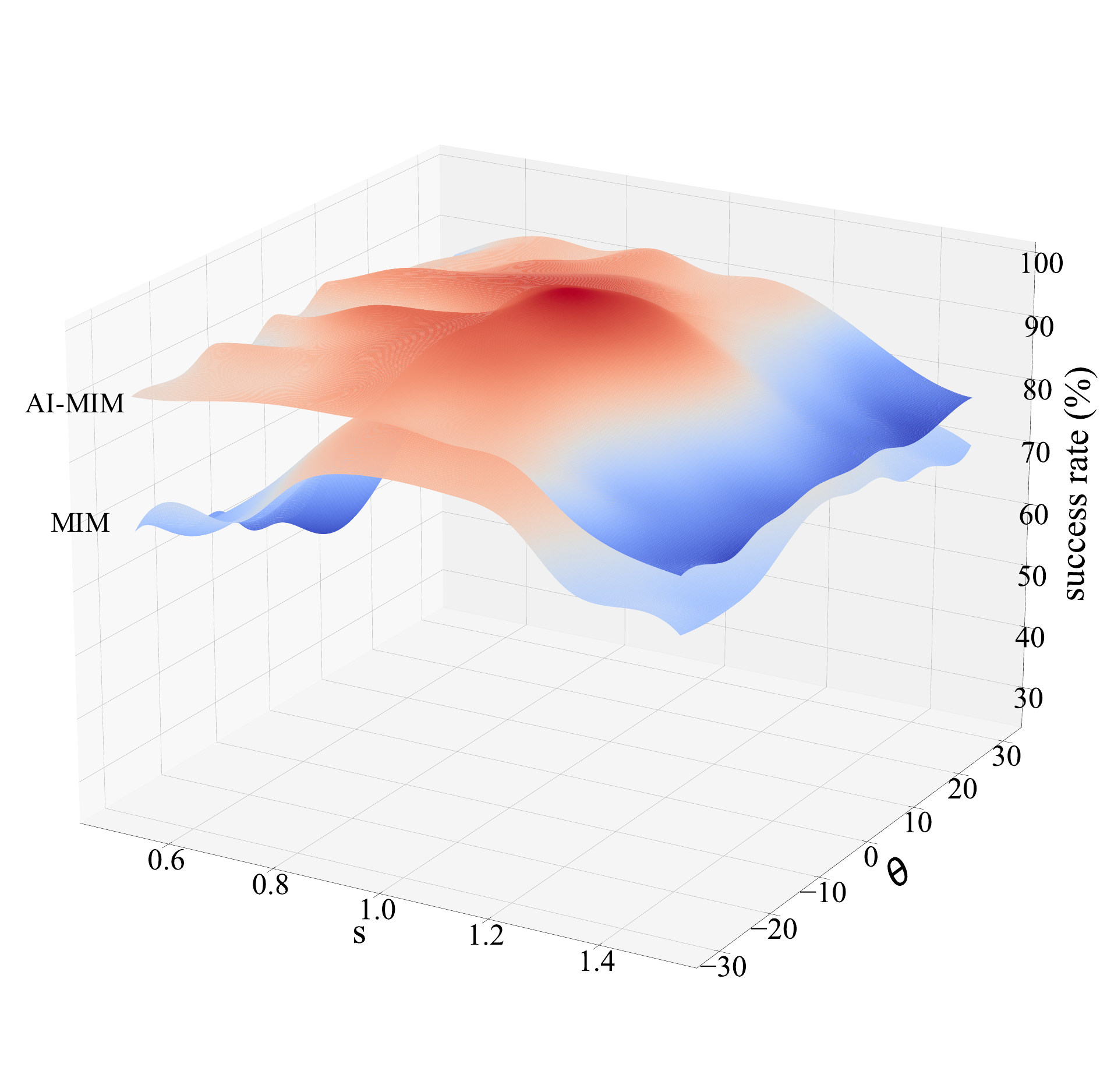} }
  \subfigure[]{\includegraphics[width=0.5\textwidth]{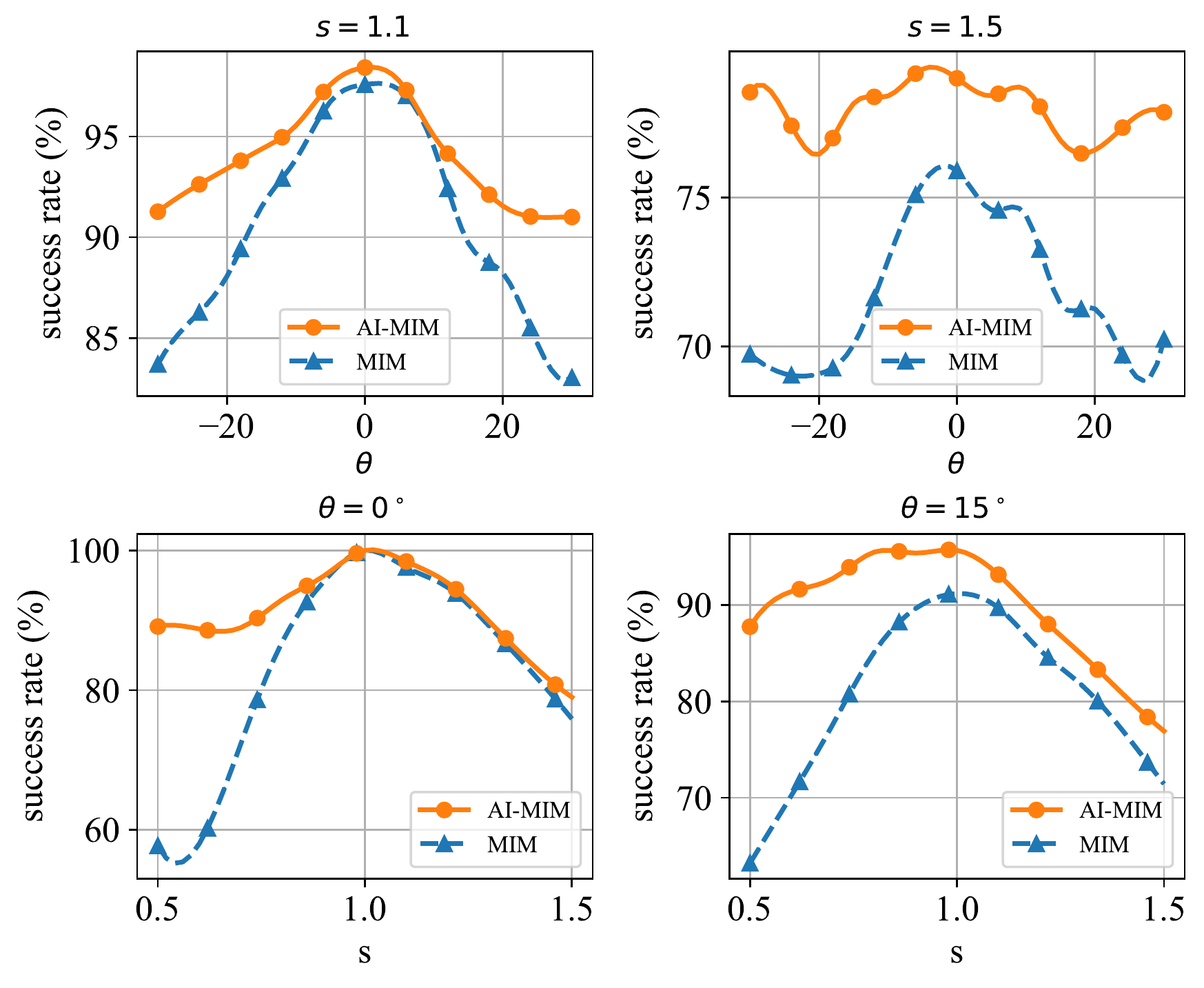} }
  \caption{The ASRs (\%) of adversarial attacks against Inc-v3 under different rotations and scalings. The adversarial examples are generated for Inc-v3 using MIM and AI-MIM. Fig. (a) shows the results in the form of a 3-D figure. Fig. (b) shows four randomly selected profiles of Fig. (a), which are $s=1.1$, $s=1.5$, $\theta=0^{\comp}$ and $\theta=15^{\comp}$.} 
  \label{fig:9} 
\end{figure*}

\begin{figure*}[htbp]
  \centering
\includegraphics[width=0.9\textwidth]{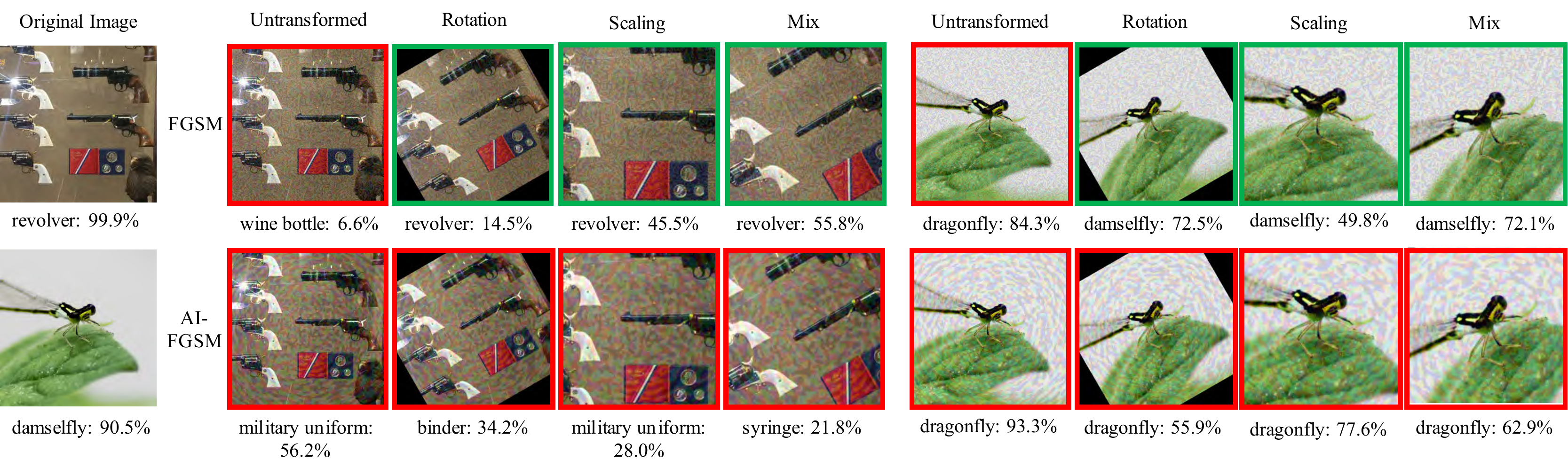}
  \caption{The adversarial examples are generated for Inc-v3 by FGSM and AI-FGSM with different transformations, including identity transform, rotation with $30^\circ$, scaling with $1.5$ factor, and mix transformation. The mix transformation consists of $30^\circ$ rotation, $1.5$ scaling factor and 20 pixels translation in both horizontal and vertical directions. In the pictures, red represents a successful attack, and green represents a failed attack.} 
  \label{fig:10} 
\end{figure*}
\begin{figure*}[htbp]
  \centering
\includegraphics[width=0.9\textwidth]{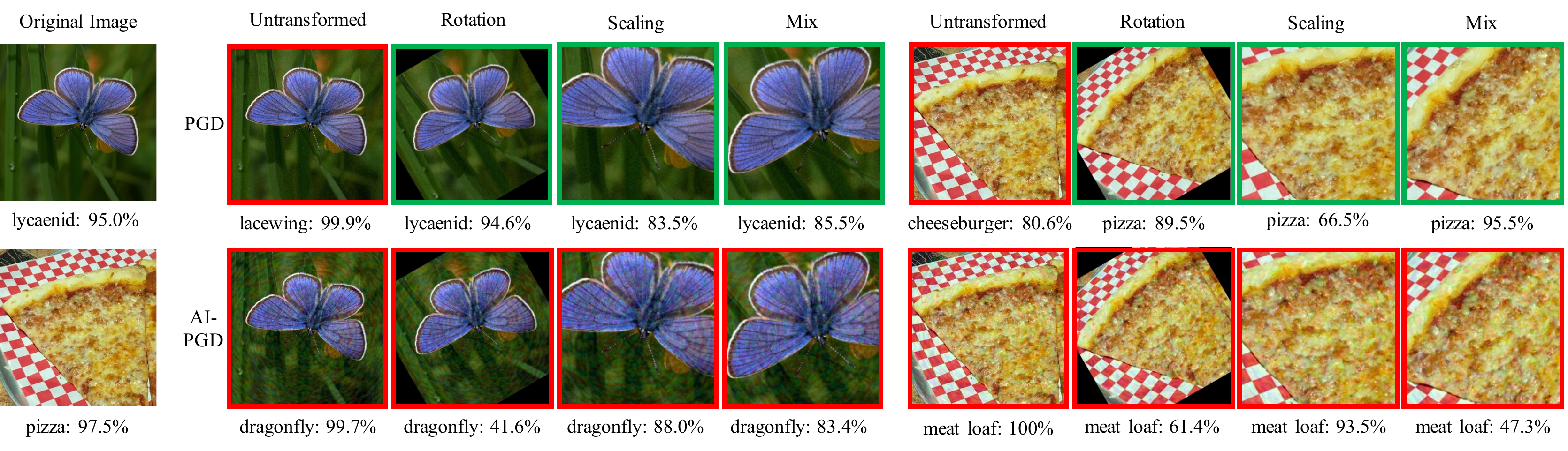}
  \caption{The adversarial examples are generated for Inc-v3 by PGD and AI-PGD with different transformations, including identity transform, rotation with $\theta=30^{\comp}$, scaling with $s = 1.5$, and mix transformation. The mix transformation consists of $\theta=30^{\comp}$, $s=1.5$ and 
  translation offsets $t=(20, 20)$. In the pictures, red represents a successful attack, and green represents a failed attack.}
  \label{fig:5} 
\end{figure*}
\begin{figure*}[htbp]
  \centering
\includegraphics[width=0.9\textwidth]{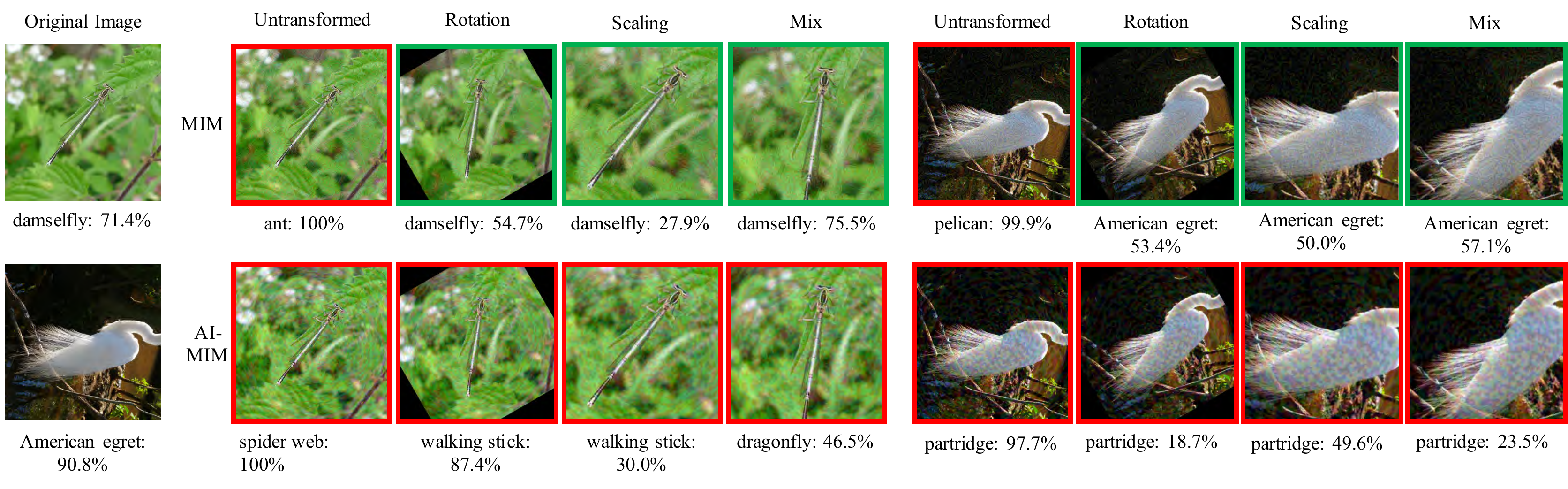}
  \caption{The adversarial examples are generated for Inc-v3 by MIM and AI-MIM with different transformations, including identity transform, rotation with $\theta=30^{\comp}$, scaling with $s = 1.5$, and mix transformation. The mix transformation consists of $\theta=30^{\comp}$, $s=1.5$ and translation offsets $t=(20, 20)$. In the pictures, red represents a successful attack, and green represents a failed attack.} 
  \label{fig:11} 
\end{figure*}

\section{Detailed Proof for Gradient Approximation}
\label{appendix:a}
Note that $\boldsymbol{g}_1 = \mathbb{E}[ \mathcal{F}_{t,q}^{-1} (\nabla_{\boldsymbol{x}} J(\boldsymbol{x}, y) \Big|_{\boldsymbol{x}=\hat{\boldsymbol{x}}} )]$ and $\boldsymbol{g}_2  = \mathbb{E}[ \mathcal{F}_{t,q}^{-1} (\nabla_{\boldsymbol{x}} J(\boldsymbol{x}, y) \Big|_{\boldsymbol{x}=\mathcal{F}_{t,q}(\hat{\boldsymbol{x}})} )]$.

\begin{corollary}
    The Euclidean distance between $\boldsymbol{g}_1$ and $\boldsymbol{g}_2$ is upper-bounded as:
    \begin{equation}
        \left\|\boldsymbol{g}_2-\boldsymbol{g}_1 \right\|_2 \leqslant c_1 \cdot c_2.
    \end{equation}
\end{corollary}

\begin{proof}
\begin{equation}
\begin{split}
    &\left\|\boldsymbol{g}_2-\boldsymbol{g}_1 \right\|_2 \\
    =& \left\|\mathbb{E}[ \mathcal{F}_{t,q}^{-1} (\nabla_{\boldsymbol{x}} J(\boldsymbol{x}, y) \Big|_{\boldsymbol{x}=\mathcal{F}_{t,q}(\hat{\boldsymbol{x}})})-\mathcal{F}_{t,q}^{-1}(\nabla_{\boldsymbol{x}} J(\boldsymbol{x}, y) \Big|_{\boldsymbol{x}=\hat{\boldsymbol{x}}} )]  \right\|_2 \\
    \leqslant& \mathbb{E}[\left\| \mathcal{F}_{t,q}^{-1} (\nabla_{\boldsymbol{x}} J(\boldsymbol{x}, y) \Big|_{\boldsymbol{x}=\mathcal{F}_{t,q}(\hat{\boldsymbol{x}})})-\mathcal{F}_{t,q}^{-1}(\nabla_{\boldsymbol{x}} J(\boldsymbol{x}, y) \Big|_{\boldsymbol{x}=\hat{\boldsymbol{x}}} )\right\|_2]  \\
        \leqslant& c_1 \cdot \mathbb{E}[\left\| \mathcal{F}_{t,q}(\hat{\boldsymbol{x}})-\hat{\boldsymbol{x}} \right\|_2] \\
    \leqslant& c_1 \cdot c_2.
\end{split}
\end{equation}
\end{proof}

\begin{corollary}
    The cosine similarity of $\boldsymbol{g}_1$ and $\boldsymbol{g}_2$ is lower-bounded as: 
    \begin{equation}
        cossim(\boldsymbol{g}_1, \boldsymbol{g}_2) \geqslant 1 - \frac{(c_1 c_2)^2}{2c_3^2},
    \end{equation}
    where $cossim(\cdot, \cdot)$ is the cosine similarity function.
\end{corollary}
\begin{proof}
    We first denote the normalized gradients as $\tilde{\boldsymbol{g}}_1=\frac{\boldsymbol{g}_1}{\left \|\boldsymbol{g}_1 \right \|_2 }$ and $\tilde{\boldsymbol{g}}_2=\frac{\boldsymbol{g}_2}{\left \|\boldsymbol{g}_2 \right \|_2 }$. From the definition of cosine similarity, we have: 
    \begin{equation}
        cossim(\boldsymbol{g}_1, \boldsymbol{g}_2)=\frac{\boldsymbol{g}_1\cdot \boldsymbol{g}_2}{\left\|\boldsymbol{g}_1 \right \|_2 \left\|\boldsymbol{g}_2 \right \|_2 }= 1-\frac{\left \| \tilde{\boldsymbol{g}}_2-\tilde{\boldsymbol{g}}_1 \right \|_2^2}{2}.
    \end{equation}
    With Eq.~(24) in Sec.~IV, we assume $\left\|\boldsymbol{g}_2 \right\|_2 \geqslant \left\|\boldsymbol{g}_1 \right\|_2 \geqslant c_3$. Then we have:
    \begin{equation}
    \begin{split}
        \left \| \tilde{\boldsymbol{g}}_2-\tilde{\boldsymbol{g}}_1 \right \|_2^2 &= \left \|\frac{\boldsymbol{g}_2}{\left \|\boldsymbol{g}_2 \right \|_2 } - \frac{\boldsymbol{g}_1}{\left \|\boldsymbol{g}_1 \right \|_2 } \right \|_2^2 \\
        &\leqslant \left \|\frac{\boldsymbol{g}_2}{\left \|\boldsymbol{g}_1 \right \|_2 } - \frac{\boldsymbol{g}_1}{\left \|\boldsymbol{g}_1 \right \|_2 } \right \|_2^2 \\
        &\leqslant \frac{1}{c_3^2}\cdot \left \| \boldsymbol{g}_2- \boldsymbol{g}_1\right \|_2^2 .
    \end{split}
    \end{equation}
    We can get a similar corollary for the case of $\left\|\boldsymbol{g}_1 \right\|_2 \geqslant \left\|\boldsymbol{g}_2 \right\|_2 \geqslant c_3$.
    Therefore, we finally have:
    \begin{equation}
    \begin{split}
        cossim(\boldsymbol{g}_1, \boldsymbol{g}_2)&=1-\frac{\left \| \tilde{\boldsymbol{g}}_2-\tilde{\boldsymbol{g}}_1 \right \|_2^2}{2} \\
        &\geqslant 1 - \frac{1}{2 c_3^2}\cdot \left \| \boldsymbol{g}_2- \boldsymbol{g}_1\right \|_2^2 \\
        &\geqslant 1-\frac{(c_1 c_2)^2}{2c_3^2} .
    \end{split}
    \end{equation}
\end{proof}

\section{Full Results For the Robustness to Affine Transformation}
\label{appendix:b}
In this section, we further visualize the white-box attack-success-rate function with rotation angle and scaling factor as independent variables for FGSM, AI-FGSM, PGD, AI-PGD and MIM, AI-MIM. Results are shown in  Fig.~\ref{fig:8} for FGSM and AI-FGSM, Fig.~\ref{fig:3} for PGD and AI-PGD, and Fig.~\ref{fig:9} for MIM and AI-MIM. The results show that our method greatly improves the attack success rate under different affine transformations, compared with the other two basic attack methods. 

\section{Full Visualization Results for transformed examples}
\label{appendix:c}
In this section, we further show adversarial images generated for the Inc-v3 model by FGSM, MIM and their extensions AI-FGSM and AI-MIM with different affine transformations. We present the visualization results of 
FGSM and AI-FGSM in Fig.~\ref{fig:10}, PGD and AI-PGD in Fig.~\ref{fig:5}, and MIM and AI-MIM in Fig.~\ref{fig:11}.
\end{appendix}

\end{document}